\documentclass[11pt]{article} 

\usepackage[american]{babel}
\usepackage{tikz}
\usepackage{pgfplots}
\pgfplotsset{compat=1.16}

\usepackage{fullpage,times,url,bm}
\usepackage{amsthm,amsfonts,amsmath,amssymb,mathtools,epsfig,color,float,graphicx,verbatim}
\usepackage{algorithm,algorithmic}
\usepackage{bbm}
\usepackage[square,numbers]{natbib}
\usepackage{epstopdf}
\usepackage{authblk}
\usepackage[utf8]{inputenc}

\usepackage{hyperref}
\hypersetup{
	colorlinks   = true,
	urlcolor     = blue,
	linkcolor    = blue,
	citecolor   = black
}

\usepackage{pgfgantt}

\usepackage{natbib} 
\usepackage{mathtools} 
\usepackage{booktabs} 
\usepackage{tikz} 
\usetikzlibrary{positioning}
\usepackage{pgfplots}
\usepackage{amsthm,amsmath,amssymb}

\newtheorem{theorem}{Theorem}[section]
\newtheorem{proposition}[theorem]{Proposition}
\newtheorem{lemma}[theorem]{Lemma}
\newtheorem{corollary}[theorem]{Corollary}
\newtheorem{definition}[theorem]{Definition}

\newtheorem{assumption}[theorem]{Assumption}
\newtheorem*{acknowledgments}{Acknowledgments}

\newcommand{\reals}{\mathbb{R}}

\newcommand{\sign}{\mathrm{sign}}

\newcommand{\relu}[1]{\left[ #1 \right]_+}
\newcommand{\set}[1]{\left\{#1\right\}}

\newcommand{\bx}{\mathbf{x}}

\newcommand{\btheta}{\boldsymbol{\theta}}

\newcommand{\Lcal}{\mathcal{L}}
\newcommand{\Ocal}{\mathcal{O}}

\newcommand{\Ncal}{\mathcal{N}}

\newcommand{\norm}[1]{\left\|#1\right\|}

\newcommand{\p}[1]{\left(#1\right)}

\newcommand{\abs}[1]{\left|#1\right|}

\newcommand{\floor}[1]{\left\lfloor#1\right\rfloor}

\DeclareMathOperator{\erfc}{erfc}

\newcommand{\secref}[1]{Section~\ref{#1}}
\newcommand{\subsecref}[1]{Subsection~\ref{#1}}
\newcommand{\figref}[1]{Figure~\ref{#1}}
\renewcommand{\eqref}[1]{Equation~(\ref{#1})}
\newcommand{\lemref}[1]{Lemma~\ref{#1}}

\newcommand{\thmref}[1]{Theorem~\ref{#1}}
\newcommand{\propref}[1]{Proposition~\ref{#1}}
\newcommand{\appref}[1]{Appendix~\ref{#1}}

\newcommand{\itemref}[1]{Item~\ref{#1}}
\newcommand{\asmref}[1]{Assumption~\ref{#1}}

\title{On the Rate of Convergence of GD in Non-linear Neural Networks:\\An Adversarial Robustness Perspective}

\author[1]{Guy Smorodinsky}
\author[2,1]{Sveta Gimpelson}
\author[1]{Itay Safran}
\affil[1]{%
    Stein Faculty of Computer and Information Science\\
    Ben-Gurion University of the Negev\\
    Israel
}
\affil[2]{%
    Department of Computer Science\\
    Reichman University\\
    Israel\\
}

\pgfmathdeclarefunction{erfcapprox}{1}{%
  \begingroup
  \pgfmathsetmacro\t{1/(1+0.3275911*#1)}%
  \pgfmathparse{(
    ((
      ((1.061405429*\t - 1.453152027)*\t + 1.421413741)*\t - 0.284496736
    )*\t + 0.254829592)*\t
  ) * exp(-(#1)^2)}%
  \pgfmathsmuggle\pgfmathresult
  \endgroup
}
\date{}
  
\begin{document}
\maketitle

\begin{abstract}
    We study the convergence dynamics of Gradient Descent (GD) in a minimal binary classification setting, consisting of a two-neuron ReLU network and two training instances. We prove that even under these strong simplifying assumptions, while GD successfully converges to an optimal robustness margin, effectively maximizing the distance between the decision boundary and the training points, this convergence occurs at a prohibitively slow rate, scaling strictly as $\Theta(1/\ln(t))$. To the best of our knowledge, this establishes the first explicit lower bound on the convergence rate of the robustness margin in a non-linear model. Through empirical simulations, we further demonstrate that this inherent failure mode is pervasive, exhibiting the exact same tight convergence rate across multiple natural network initializations. Our theoretical guarantees are derived via a rigorous analysis of the GD trajectories across the distinct activation patterns of the model. Specifically, we develop tight control over the system's dynamics to bound the trajectory of the decision boundary, overcoming the primary technical challenge introduced by the non-linear nature of the architecture.
\end{abstract}

\section{Introduction}\label{sec:intro}

The implicit bias of optimization algorithms has become a central pillar of deep learning theory in the effort to better understand the generalization capabilities of neural networks. It is now well-understood that continuing optimization long after achieving perfect classification can significantly improve standard generalization performance. Furthermore, research has increasingly shown that this over-training can also enhance a model's adversarial robustness---its resilience against imperceptible perturbations designed to induce misclassification \citep{biggio2013evasion,szegedy2013intriguing}. Gradient descent (GD) remains the fundamental algorithm for training these models; however, its theoretical analysis is often confounded by the non-differentiability of ReLU activations and the complexities of discrete step-size tuning. To bypass these technical hurdles, many researchers employ gradient flow (GF) as a continuous-time surrogate. While GF yields more amenable analyses, it often abstracts away the discrete dynamics of actual training. Thus, determining the precise convergence rates of its discrete-time counterpart, GD, remains a critical and non-trivial challenge.

Unlike its discrete counterpart, the implicit bias of GF is significantly better characterized. In the context of homogeneous networks, GF has been shown to maximize the classification margin in the sense of the Karush-Kuhn-Tucker (KKT) conditions \citep{lyu2019gradient,ji2020directional}. Apart from improving generalization \citep{lyu2021gradient,sarussi2021towards,frei2022implicit}, because this geometric margin dictates the distance to the decision boundary, such margin-maximizing behavior is widely hypothesized to also foster natural resilience against adversarial perturbations. Building on this intuition, recent work has explicitly guaranteed that gradient flow can inherently drive certain networks toward adversarially robust solutions \citep{min2025gradient}.

However, it remains an open question whether GD inherently exhibits the same bias. Even when explicit robustness guarantees are established for GF, and in specific regimes where GD and GF are known to align \citep{elkabetz2021continuous}, GF fundamentally abstracts away computational constraints. Consequently, it remains unclear whether practical algorithms like GD can converge to these robust optima within a tractable number of iterations. Moreover, the introduction of adaptive methods such as Adam complicates this landscape, as their distinct implicit biases may render robustness guarantees derived from GF analysis inapplicable. Despite ample evidence that the convergence rate for linear models is of order $\Theta(1/\ln(t))$ \citep{soudry2018implicit,ji2018risk,nacson2019stochastic,nacson2019convergence,ji2021characterizing,ji2021fast,wang2024achieving}, determining the convergence rate in non-linear models, which are of primary interest from a practical perspective, remains a fundamental open problem.

In this paper, we seek to bridge these gaps by studying the convergence of GD toward an adversarially robust margin. We focus on a depth-2, width-2 ReLU network of the form
\begin{equation}\label{eq:architecture}
    \Phi\p{\btheta;x} \coloneqq \relu{w_1x+b_1} - \relu{w_2x+b_2},
\end{equation}
where the weights of the output neuron are fixed, and optimization is performed over the hidden layer parameters $\btheta = \{w_1, b_1, w_2, b_2\}$. We analyze the optimization dynamics of GD and GF on a binary classification task with training samples $(x_1,y_1)=(-1,-1)$ and $(x_2,y_2)=(1,1)$. The objective is the empirical risk under the exponential loss, $\ell(z) \coloneqq \exp(-z)$,\footnote{As the logistic loss $z \mapsto \ln(1+\exp(-z))$ exhibits the same exponential tail decay, our results can be naturally extended to that setting; however, we focus on the exponential loss for clarity and analytical simplicity.} a standard choice for studying implicit bias in classification:
\begin{equation}\label{eq:obj}
    \Lcal\p{\btheta} \coloneqq \frac12\exp\p{\Phi\p{\btheta;-1}} + \frac12\exp\p{-\Phi\p{\btheta;1}}.
\end{equation}

In this paper, we demonstrate that while GF converges to robust solutions that maximize the margin as $t \to \infty$ in this simplified setting, these results provide a false sense of security for practical optimization. We show that while GD eventually converges to the same robust solutions, the rate of convergence is almost always prohibitively slow: the distance to the optimal margin decays only at a rate of $\Theta(1/\ln(t))$. To the best of our knowledge, this provides the first provable slow convergence guarantee for the robustness margin in a non-linear neural network, and significantly strengthens previous findings which were restricted to linear models. 

Furthermore, we complement our theoretical analysis with empirical evaluations. Our experiments reveal that our derived theoretical rate is, in fact, conservative; due to necessary analytical relaxations made to tractably bound the dynamics, practical convergence to a robust margin is often even slower. Because our analysis provides a negative result in a highly restricted setting, these simplifying assumptions serve to underscore the severity of the bottleneck. If convergence is prohibitively slow in a minimal architecture, it is highly improbable that it will accelerate in complex, over-parameterized models, where these exact slow dynamics can easily emerge as a constituent sub-network.

The remainder of this paper is structured as follows: After surveying additional related work, we establish our notation and provide the necessary theoretical background in \secref{sec:prelims}. In \secref{sec:GF}, we analyze the convergence of GF toward a robust network within our specified setting. \secref{sec:GD_dynamics} details the optimization dynamics of discrete-time GD, providing the technical foundation for \secref{sec:slow_convergence}, where we derive our main result: the $\Theta(1/\ln(t))$ convergence rate to the optimal robustness margin. Finally, in \secref{sec:experiments}, we present empirical findings that validate our theoretical bounds and demonstrate that convergence under standard initialization schemes is frequently even slower than our theoretical predictions.


\subsection{Related work}

\paragraph{Implicit Bias Towards Robust Solutions}

The seminal works of \citet{lyu2019gradient} and \citet{ji2020directional} provide a strong characterization of the implicit bias of GF in homogeneous neural networks. Specifically, \citet{lyu2019gradient} link the implicit bias in non-linear networks to margin maximization. They prove that for ReLU networks, GF converges in direction to a KKT point of the maximum margin problem (see \subsecref{subsec:KKT_prelims} for a formal definition), a geometric property that inherently corresponds to pushing the network's decision boundary farther away from the training instances.

Building on this, \citet{ji2020directional} utilize Kurdyka-\L{}ojasiewicz inequalities to analyze deep homogeneous networks, demonstrating that the network's parameters and internal features directionally align to maximize this margin, thereby guaranteeing that the margin distribution asymptotically converges. While these topological approaches suggest an implicit regularization toward improved adversarial robustness, they do not yield explicit finite-time convergence guarantees. 

Recently, \citet{min2025gradient} formalized this connection between implicit bias and adversarial robustness by proving that for shallow networks trained on clustered data, the use of polynomial ReLU (pReLU) activations allows gradient flow to implicitly find solutions with certified $\Ocal(1)$ adversarial robustness. However, because this guarantee pertains exclusively to the continuous-time dynamics of GF, it inherently disregards computational constraints.

Despite the optimistic guarantees discussed above, the connection between margin maximization in the KKT sense and adversarial robustness is not absolute. For instance, \citet{frei2023double} and \citet{melamed2023adversarial} demonstrate that there exist settings where the implicit bias of GF leads to KKT points that are provably suboptimal in terms of their robustness margin. While these works highlight failure modes in the asymptotic direction of GF, our work exposes an orthogonal computational vulnerability: even in settings where the implicit bias perfectly aligns with the optimal robust decision boundary, the finite-time dynamics of GD prevent convergence in tractable time.

\paragraph{Slow Margin Convergence}

As noted in the introduction, there is ample work studying the rate of convergence to an optimal margin \citep{soudry2018implicit,ji2018risk,lyu2019gradient,nacson2019stochastic,nacson2019convergence,ji2021characterizing,ji2021fast,wang2024achieving}. While these papers are foundational for the modern understanding of implicit bias, their core quantitative results regarding the convergence rate of the weight vector to the max-margin direction are predominantly derived for linear classifiers. Moreover, these analyses typically bound the growth rate of the weights, which in linear models inherently dictates a slow, logarithmic convergence to the optimal margin. In contrast, our setting is significantly more technically involved. In non-linear models, a logarithmic growth bound on the weights does not trivially preclude fast convergence of the robustness margin, as the non-linearities and bias terms could theoretically allow the decision boundary to rapidly shift to the optimal location. Consequently, existing analyses based solely on parameter growth are inadequate for establishing convergence lower bounds in our setting.

\paragraph{Acceleration Algorithms to Improve Convergence Rate}
\citet{nacson2019convergence}, \citet{ji2021fast}, and \citet{wang2024achieving} propose alternatives to standard GD designed to accelerate convergence to the optimal margin in the linear setting. Specifically, \citet{nacson2019convergence} and \citet{ji2021fast} prove that NGD, a variant of GD with a step size scaled by the training loss, converges at a faster rate of $\Ocal(1/t)$. \citet{wang2024achieving} further improve this rate to $\Ocal(\exp(-t))$ by introducing a novel algorithm, PRGD, that progressively rescales the network parameters to artificially bypass the logarithmic norm growth bottleneck of standard GD. Despite these vastly improved convergence rates, it is not obvious whether such algorithms necessarily enjoy the same implicit bias benefits as GF or unaccelerated GD. Moreover, because these methods scale the step size by a significant quantity, they are highly susceptible to ``overshooting'' in non-linear loss landscapes. In such scenarios, an excessively large step size can prevent convergence entirely or severely impede optimization by permanently deactivating ReLU neurons.

\section{Preliminaries and Notation}\label{sec:prelims}

\subsection{Notation and Terminology}

We use bold-faced letters to denote vectors (e.g., $\bx$) and capital letters to denote random variables (e.g., $X$). We let $\mathcal{N}(\mu, \sigma^2)$ denote a normal distribution with mean $\mu$ and variance $\sigma^2$. For a model $\Phi(\btheta; \cdot): \mathbb{R}^d \to \mathbb{R}$, we let the predicted class of a point $\bx$ be given by $\sign(\Phi(\btheta; \bx))$. The robustness margin $\gamma(\btheta)$ with respect to a dataset $\{(\bx_i, y_i)\}_{i=1}^n$ is the minimum distance required to perturb an instance $\bx_i$ such that the model's prediction changes:
\[
    \gamma(\btheta) \coloneqq \min_{i \in \{1, \dots, n\}} \inf \{ \| \Delta \| : \sign(\Phi(\btheta; \bx_i + \Delta)) \neq y_i \}.
\]
For the target architecture in \eqref{eq:architecture}, we let $x^\star$ denote the decision threshold where the model's prediction flips. When this point is unique, the robustness margin simplifies to the distance from the training instances to $x^\star$, namely $\gamma = \min\{|x^\star - 1|, |x^\star+ 1|\}$. The optimal margin, denoted $\gamma^\star$, is attained when $x^\star = 0$, yielding $\gamma^\star = 1$. A network $\Phi(\btheta; \bx)$ is called \emph{homogeneous} if there exists $\alpha > 0$ such that for all $c > 0$, all $\btheta$ and all $\bx$, it holds that $\Phi(c \cdot \btheta; \bx) = c^\alpha\Phi(\btheta; \bx)$. We utilize standard asymptotic notation (e.g., $\mathcal{O}, \Omega, \Theta$).

\subsection{Preliminaries on KKT Conditions}\label{subsec:KKT_prelims}

    \begin{definition}[Directional convergence]
        We say that GF or GD \emph{converge in direction} to $\hat{\btheta}$ if
        \[
            \lim_{t \rightarrow \infty}\frac{\btheta_t}{\|\btheta_t\|} = \frac{\hat{\btheta}}{\|\hat{\btheta}\|},
        \]
        where $\btheta_t=\btheta(t)$ denotes the flow at time $t$ for GF, and $\btheta_t=\btheta^{(t)}$ denotes the $t^{\textrm{th}}$ iterate of GD, depending on the optimization algorithm.
    \end{definition}

    The following seminal result characterizes the implicit bias of GF on homogeneous neural networks.
    \begin{theorem}[\citet{lyu2019gradient,ji2020directional}] \label{thm:known_KKT}
        Let $\Phi(\btheta;\cdot)$ be a depth-$2$ ReLU neural network parameterized by 
        $\btheta$.  
        Consider minimizing either the exponential or the logistic loss over a binary 
        classification dataset $\{(x_i,y_i)\}_{i=1}^n$ using GF.  
        Assume that there exists a time $t_0$ such that 
        $\Lcal(\btheta(t_0)) < \tfrac{1}{n}$, i.e., 
        $y_i \Phi(\btheta(t_0); x_i) > 0$ for every $x_i$.  
        Then, GF converges in direction to a first-order stationary point 
        (a KKT point) of the following max-margin problem in parameter space:
        \begin{equation} \label{eq:optimization problem}
            \min_{\btheta} \; \frac{1}{2}\|\btheta\|^2
            \quad \text{s.t.} \quad 
            y_i\,\Phi(\btheta;x_i) \ge 1 \;\; \text{for all } i \in [n].
        \end{equation}
        Moreover, $\Lcal(\btheta(t)) \to 0$ and $\|\btheta(t)\| \to \infty$ as 
        $t \to \infty$.
    \end{theorem}

    When its conditions are satisfies, the above theorem implies that the model converges in direction to a parameter vector $\btheta$ that satisfies the following explicit KKT constraints:
    \begin{align}
        &\btheta = \sum_{i=1}^n\lambda_iy_i\nabla_{\btheta}\Phi(\btheta;\bx_i),\label{eq:theta}\\
        &y_i\Phi(\btheta;\bx_i)\ge 1~~\text{for all }i\in[n],\label{eq:correct}\\
        &\lambda_1,\dots,\lambda_n \geq 0, \label{eq:primal_feasibility}\\
        &\text{for all }i\in[n],~~\text{if}~~y_i\Phi(\btheta;\bx_i)\neq 1 ~ \text{ then } \lambda_i=0. \label{eq:zero_lam}
    \end{align}

\section{Convergence of GF to Robust Solutions}\label{sec:GF}

    We begin by characterizing the implicit bias of GF on the objective defined in \eqref{eq:obj}. As discussed in the introduction, recent literature has established that for homogeneous networks,\footnote{It is straightforward to prove that the architecture we study in \eqref{eq:architecture} is homogeneous with $\alpha=1$.} the implicit bias of GF drives the parameters toward a solution that maximizes the classification margin in the KKT sense. By applying \thmref{thm:known_KKT}, we demonstrate that our specific architecture also adheres to this behavior, converging to a solution with an optimal robustness margin.
    \begin{theorem}\label{thm:KKT}
        Consider optimizing the objective in \eqref{eq:obj} using GF. Assume that there exists a time $t_0\ge0$ such that $\Lcal(\btheta(t_0))<0.5$. Then,
        \begin{equation}\label{eq:KKT_implciit_bias}
            \lim_{t\to\infty}\frac{\btheta(t)}{\norm{\btheta(t)}} = \btheta^\star\coloneqq(0.5,0.5,-0.5,0.5).
        \end{equation}
        In particular, $\btheta^\star$ achieves an optimal robustness margin $\gamma(\btheta^\star)=\gamma^\star$.
    \end{theorem}

    The proof of Theorem~\ref{thm:KKT}, provided in \appref{app:KKT_proof}, relies on an exhaustive case-by-case analysis of the parameter space to identify solutions satisfying the KKT conditions in Equations~(\ref{eq:theta}--\ref{eq:zero_lam}). Our analysis demonstrates that the solution is unique up to a positive linear rescaling, providing a precise characterization of the directional convergence of GF. Furthermore, as illustrated in \figref{fig:limiting_network}, the network converges to a configuration achieving the optimal robustness margin $\gamma^\star = 1$. This convergence behavior suggests that GD might similarly benefit from such an implicit bias, eventually driving the network toward an optimal robust solution. Motivated by this possibility, we analyze the discrete-time dynamics in the following section.

    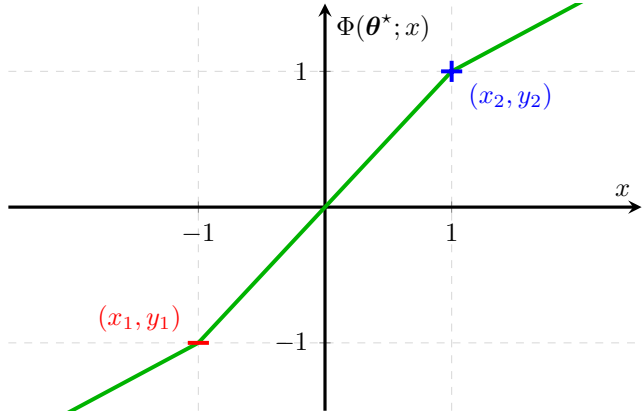
\begin{figure}[ht]
        \centering
        \begin{tikzpicture}
            \begin{axis}[
                width=10cm,
                height=7cm,
                axis lines=middle,
                xlabel={$x$},
                ylabel={$\Phi(\btheta^\star; x)$},
                xmin=-2.5, xmax=2.5,
                ymin=-1.5, ymax=1.5,
                xtick={-1, 1},
                ytick={-1, 1},
                grid=both,
                grid style={dashed, gray!30},
                very thick,
                tick label style={font=\small},
                label style={font=\small}
            ]
                \addplot[
                    domain=-2.5:2.5, 
                    samples=200, 
                    green!70!black, 
                    line width=1.5pt
                ] {max(0, 0.5*x + 0.5) - max(0, -0.5*x + 0.5)};
                
                \addplot[
                    only marks, 
                    mark=+, 
                    blue, 
                    mark size=4pt, 
                    line width=1.5pt
                ] coordinates {(1,1)}
                node[anchor=north west, xshift=2pt, font=\small] {$(x_2, y_2)$};
                
                \addplot[
                    only marks, 
                    mark=-, 
                    red, 
                    mark size=4pt, 
                    line width=1.5pt
                ] coordinates {(-1,-1)}
                node[anchor=south east, xshift=-2pt, font=\small] {$(x_1, y_1)$};
                
            \end{axis}
        \end{tikzpicture}
        \caption{The limiting network $\Phi(\btheta^\star; x)$ for $\btheta^\star = (0.5, 0.5, -0.5, 0.5)$. The network achieves an optimal robustness margin $\gamma^\star = 1$, perfectly separating the training samples $(x_1, y_1)$ and $(x_2, y_2)$ indicated by the red and blue markers, respectively.}
        \label{fig:limiting_network}
    \end{figure}

\section{The Optimization Dynamics of GD}\label{sec:GD_dynamics}

    Having characterized the implicit bias of GF, we now analyze the robustness properties of its discrete-time counterpart, GD. To do so, we must first establish the underlying optimization dynamics. Throughout this section, we assume that $\Lcal(\btheta^{(0)}) < 0.5$. In our setting, this condition aligns our analysis with the assumptions of Theorems~\ref{thm:known_KKT} and \ref{thm:KKT}, and implies that all data instances are correctly classified, which is a necessary intermediate state for reaching vanishing training error. Given that the primary motivation for studying implicit bias is to understand the beneficial properties emerging as the loss approaches zero, it is natural to focus our analysis on the regime following perfect classification. By investigating the dynamics in this stage, we can precisely quantify how the robustness guarantees discussed in the introduction are (or are not) efficiently realized by GD.

\subsection{Neuron specialization}

    We begin by analyzing the activation patterns once the training loss in \eqref{eq:obj} drops below $0.5$. this constraint ensures that each neuron classifies its corresponding data instance correctly: specifically, the instance whose label matches the sign of the neuron's contribution to the network output. This significantly restricts the set of possible activation patterns. The subsequent dynamics are governed by the fact that GD updates drive each neuron to ``specialize" on its corresponding data instance. As optimization progresses, each neuron eventually becomes inactive for the oppositely labeled instance. These dynamics are visualized in \figref{fig:dynamics_flowchart}.

    \begin{figure}[ht]
        \centering
        \begin{tikzpicture}[
            state/.style={rectangle, draw, rounded corners, inner sep=6pt, text width=0.27\columnwidth, align=center, font=\small, minimum height=2cm},
            arrow/.style={-stealth, line width=2.5pt},
            skiparrow/.style={-stealth, line width=2.5pt, dashed}
        ]
    
        \node[state] (initial) at (0,0) {
            \textbf{Linear setting} \\  
            \vspace{2pt}\rule{0.8\linewidth}{0.4pt}\vspace{2pt} \\ 
            Both neurons active\\ on both instances.
        };
    
        \node[state] (transition) at (5.8,0) {
            \textbf{One neuron specializes} \\
            \vspace{2pt}\rule{0.8\linewidth}{0.4pt}\vspace{2pt} \\
            One neuron active on both,\\ the other on only one.
        };
    
        \node[state] (final) at (11.6,0) {
            \textbf{Both neurons specialize} \\
            \vspace{2pt}\rule{0.8\linewidth}{0.4pt}\vspace{2pt} \\
            One neuron active on $x_1$,\\ the other active on $x_2$.
        };
    
        \draw[arrow] (initial) -- (transition);
        \draw[arrow] (transition) -- (final);
        
        \draw[arrow] (initial.north) to[bend left=24] node[midway, above, align=center, font=\scriptsize] {} (final.north);
    
        \end{tikzpicture}
        \caption{Simplified optimization dynamics after the loss decreases below $0.5$. 
        The training trajectory transitions between different states as depicted by the figure. The intermediate state can bifurcate depending on the initial orientations of the neurons. Moreover, in the final state, a degenerate interval may arise in which no neuron is active; this case requires a separate technical analysis to establish that after a finite number of iterations, the degenerate interval vanishes.}
        \label{fig:dynamics_flowchart}
    \end{figure}
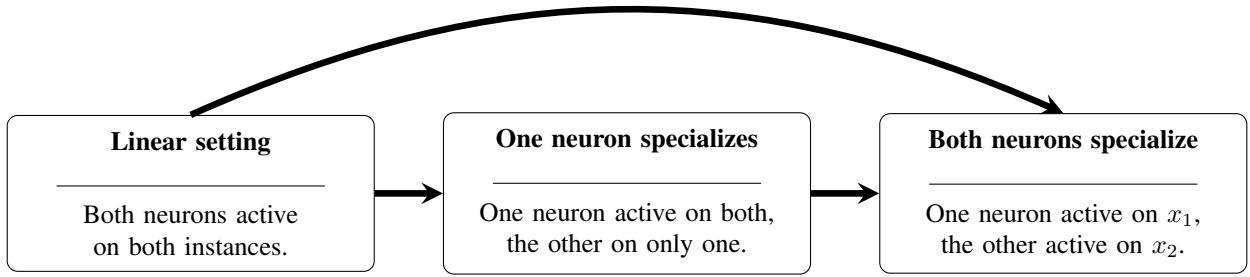

\subsection{Optimization dynamics at specialization}

    Once each neuron specializes on a distinct data instance, the optimization dynamics become more analytically tractable. Specifically, we prove that the activation pattern attained at this stage is invariant under GD updates for any sufficiently small step size, and the dynamics are governed by the following explicit update rules:
    \begin{align}
        w_1^{(t+1)} &= w_1^{(t)} + \frac{1}{2}\eta \exp\p{-w_1^{(t)} - b_1^{(t)}}, \label{eq:w_1_rule_body}\\
        b_1^{(t+1)} &= b_1^{(t)} + \frac{1}{2}\eta \exp\p{-w_1^{(t)} - b_1^{(t)}}, \\
        w_2^{(t+1)} &= w_2^{(t)} - \frac{1}{2}\eta \exp\p{w_2^{(t)} - b_2^{(t)}}, \\
        b_2^{(t+1)} &= b_2^{(t)} + \frac{1}{2}\eta \exp\p{w_2^{(t)} - b_2^{(t)}}. \label{eq:b_2_rule_body}
    \end{align}
    These rules yield several key insights. First, the weights and biases are monotone, with their signs asymptotically aligning with the limiting direction of $\btheta^\star$, defined in \eqref{eq:KKT_implciit_bias}. Second, the activation threshold (or breakpoint) of each neuron is pushed toward the instance on which it has become inactive. Third, each neuron's weight and bias move in tandem, whereas the two neurons can update at completely different magnitudes. As we explore in the next section, the imbalanced initialization of these parameters is the fundamental driver of the slow convergence to an optimal robustness margin, as it determines the specific limiting behavior and the resulting discrepancy in the update magnitudes.

\subsection{Reaching a GD update equilibrium}

    Through a careful analysis of the GD dynamics, we prove that the system eventually converges to an equilibrium state. More concretely, we establish the following limiting behavior:
    \begin{align}
        &\lim _{t \to \infty} b_2^{(t)} - b_1^{(t)} = \lim _{t \to \infty} w_1^{(t)} + w_2^{(t)} = \frac{w_1^{(0)} - b_1^{(0)} + w_2^{(0)} + b_2^{(0)}}{2}.\label{eq:bias_limit_body}
    \end{align}
    This implies the system reaches an equilibrium where $b_2^{(t)} - b_1^{(t)} \approx w_1^{(t)} + w_2^{(t)}$. Specifically, this entails $-w_1^{(t)} - b_1^{(t)} \approx w_2^{(t)} - b_2^{(t)}$, ensuring that the gradient updates in Equations~(\ref{eq:w_1_rule_body}--\ref{eq:b_2_rule_body}) become balanced across both neurons.
    
    Beyond synchronizing the neurons' updates, this result allows us to show that GD shares the same implicit bias as GF in this setting, converging exclusively to robust solutions.\footnote{This follows from showing that when the update rule is approximately identical for all parameters, the normalized parameter vector $\btheta^{(t)} / \norm{\btheta^{(t)}}$ converges to $\btheta^\star$, as defined in \eqref{eq:KKT_implciit_bias}.} Crucially, this equilibrium provides the tight control over the parameters necessary to characterize the model's robustness margin and its rate of convergence.
    
\section{Exponentially Slow Convergence to Robust Solutions}\label{sec:slow_convergence}

    In this section, we leverage the analytical machinery developed in \secref{sec:GD_dynamics} to derive the convergence rate of GD toward a robust solution.
    
    \subsection{Computational Intractability of Margin Minimization with GD}\label{subsec:almost_surely_slow}
    
    We first demonstrate that once both neurons specialize and their activation thresholds are sufficiently close to the instances for which they are inactive, there exists a unique intersection point between the network output and the $x$-axis. This point, denoted by $x^\star(t)$ at iteration $t$, determines the robustness margin of our model and is defined by:
    \begin{equation}\label{eq:x_star_body}
        x^\star(t) = \frac{b_2^{(t)} - b_1^{(t)}}{w_1^{(t)} - w_2^{(t)}}.
    \end{equation}

    According to \eqref{eq:bias_limit_body}, the numerator in our expression for $x^\star(t)$ converges to an initialization-dependent constant. Given that neuron specialization implies $w_2^{(t)} < 0 < w_1^{(t)}$, the denominator remains strictly positive. By further characterizing the growth rate of the weights, we demonstrate that their magnitudes are $\Theta(\ln(t))$. Consequently, unless $w_1^{(0)} - b_1^{(0)} + w_2^{(0)} + b_2^{(0)} = 0$ (a condition that occurs only on an initialization set of Lebesgue measure zero), the convergence rate of $x^\star(t)$ to the origin is $\Omega(1/\ln(t))$. More formally, we establish the following main result.

    \begin{theorem}\label{thm:no_PAC}
        Consider the optimization of the objective in \eqref{eq:obj} using GD with a fixed step size $\eta \in (0,0.5]$. Then, for almost all initializations $\btheta^{(0)}$ satisfying $\Lcal(\btheta^{(0)}) < 0.5$, we have.
        \[
            \gamma^*-\gamma\p{\btheta^{(t)}} = \Theta\p{\frac{1}{\ln(t)}}.
        \]
    \end{theorem}
    The proof of the above theorem, deferred to \appref{app:no_PAC_proof}, leverages the machinery developed in the previous section to provide a comprehensive characterization of the optimization dynamics within the small-loss regime. Specifically, we establish the eventual convergence of GD while demonstrating that the asymptotic limit governing the decision boundary remains permanently biased away from the optimal robustness margin.

    The above result establishes that even under extremely mild assumptions, GD on homogeneous neural networks cannot efficiently optimize the robustness margin. Crucially, this logarithmic rate bottleneck is inherent to the optimization dynamics: the probability of encountering such slow convergence is not merely positive, but equals $1$ under any continuous distribution over the initial parameters. This implies that reaching a sufficiently small desired robustness threshold $\varepsilon>0$ requires a number of iterations that scales exponentially with $1/\varepsilon$.

    \subsection{Non-asymptotic Analysis}

    While the preceding theorem establishes a strong negative result, its asymptotic nature leaves open a practical loophole: it could be that achieving a margin arbitrarily close to the optimum requires prohibitive computational resources, but reaching a practical near-optimal margin is relatively easy. To rule out this possibility, we provide a non-asymptotic analysis of the model's robustness margin under conditions commonly satisfied by standard initialization schemes \citep{he2015delving}. Concretely, our analysis reveals that under these assumptions, the convergence toward the equilibrium limit in \eqref{eq:bias_limit_body} is strictly monotone. Because He initialization typically initializes biases to zero, the model trivially begins with an optimal margin. However, any initial imbalance in the weights triggers the dynamics discussed in the previous section, actively driving the margin away from the optimum at a linear rate.
    
    Finally, once the system reaches a step-size equilibrium, the margin begins to gradually improve toward the optimum, but it does so at a prohibitively slow rate.
    
    Formally, we establish the following non-asymptotic lower bound on the robustness margin, demonstrating its prohibitively slow recovery after the initial degradation phase.
    \begin{theorem}\label{thm:non_asymptotic}
        Consider optimizing the objective in \eqref{eq:obj} using GD with step size $\eta\in(0,0.3]$. Suppose that $w_1^{(0)}\in(2.7,3)$, $w_2^{(0)}\in(-0.5,0)$, and that $b_1^{(0)}=b_2^{(0)}=0$. Then, for all $t\ge7/\eta$, we have
        \[
            x^\star(t) \ge \frac{1}{25+5\ln(1+4\eta t)}.
        \]
    \end{theorem}
    The proof of the above theorem, deferred to \appref{app:proof_non_asymptotic}, relies on a careful analysis of the decision boundary dynamics under the specified initialization conditions. Notably, the primary technical challenge arises because merely bounding the individual growth rates of the weights and biases is insufficient for establishing tight control over the decision boundary. Instead, we explicitly track the trajectory of the parameters to exploit the zero-initialization of the biases alongside the initial discrepancy between the weights. This allows us to prove that during the early iterations of the algorithm, the second bias term ``escapes'' the first, permanently establishing a lower bound on their difference that drives the $\Omega(1/\ln(t))$ bottleneck.
    
    As a direct consequence of Theorem \ref{thm:non_asymptotic}, we establish the following corollary for standard initialization schemes.
    \begin{corollary}
        Consider optimizing the objective in \eqref{eq:obj} using GD with step size $\eta\in(0,0.3]$, initializing the network parameters via standard He initialization ($w_1^{(0)},w_2^{(0)}\sim\Ncal(0,2)$ and $b_1^{(0)}=b_2^{(0)}=0$). Then, with constant probability over the random initialization, for all iterations $t\ge7/\eta$, the decision boundary satisfies
        \[
            x^\star(t) \ge \frac{1}{25+5\ln(1+4\eta t)}.
        \]
    \end{corollary}
    Consequently, our model cannot efficiently return a robust predictor with arbitrarily high probability under standard gradient descent. The initialization exhibits a constant probability of trapping the dynamics in a logarithmically slow regime; thus, unlike our previous asymptotic analysis in \subsecref{subsec:almost_surely_slow}, we establish that the intersection point is strictly bounded away from the optimum for all moderately large $t$.

\section{Experiments: The Convergence Rate to a Stable System}\label{sec:experiments}

In this section, we complement our theoretical findings with empirical studies. We investigate the number of iterations required to reach a near-equilibrium state, and we demonstrate that the convergence rates obtained in practice are even more pessimistic than what our theory predicts.

Since our analysis is asymptotic in nature, it does not explicitly capture the transient rate at which the system converges to the equilibrium governed by the updates in Equations~(\ref{eq:w_1_rule_body}--\ref{eq:b_2_rule_body}). Moreover, while \eqref{eq:bias_limit_body} dictates the limiting behavior of the numerator in the decision boundary expression \eqref{eq:x_star_body}, the speed of this convergence is not obvious. In particular, it is possible that while the denominator grows slowly, it could still be much larger than the numerator during early training, potentially keeping the distance from the optimal margin small for a significant duration.

We now turn to study this rate empirically to demonstrate that the convergence to equilibrium is sufficiently rapid, ensuring that the overall margin convergence is primarily governed by the slow growth rate of the denominator. To this end, we conducted $10,000$ independent trials of GD with a fixed step size $\eta=0.05$ to optimize the objective in \eqref{eq:obj}. We employed standard He initialization \citep{he2015delving}, setting $b_1^{(0)}=b_2^{(0)}=0$ and drawing weights from a normal distribution, $w_1^{(0)}, w_2^{(0)} \sim \mathcal{N}(0,2)$.

Optimization continued unless a trial failed to achieve a training error below $0.5$ within $10,000$ iterations. For successful trials, we terminated the process once the equilibrium condition was met: $|w_1^{(t)} + w_2^{(t)} + b_1^{(t)} - b_2^{(t)}| \le 0.001$. This simulation reveals that a necessary condition for correctly classifying both data instances under this initialization is that the weights are initialized with appropriate signs: $w_2^{(0)} < 0 < w_1^{(0)}$. Consequently, only about a quarter of the trials (specifically precisely those where the initialization fell into the correct quadrant) were able to train beyond the initial $10,000$ iterations.

In all successful trials, GD eventually reached the equilibrium point where $|w_1^{(t)} + w_2^{(t)} + b_1^{(t)} - b_2^{(t)}| \le 0.001$, confirming that the system always stabilizes. While most runs converged to this threshold within a few hundreds of thousands of iterations, exceptional cases existed where a significant initial discrepancy in weight magnitudes led to a much slower convergence rate. However, because these are precisely the cases where the bias difference is largest, the convergence to an optimal margin remains significantly slower than in more balanced, stable cases where the difference is more modest.

Under our initialization assumptions, \eqref{eq:bias_limit_body} implies
\[
    \lim_{t\to\infty} b_2^{(t)} - b_1^{(t)} = \frac{1}{2}w_1^{(0)} + \frac{1}{2}w_2^{(0)}.
\]
Conditioning on each weight having the correct sign for specialization at initialization ($w_1^{(0)} > 0$ and $w_2^{(0)} < 0$), we let $W_1 = w_1^{(0)}$ and $W_2 = -w_2^{(0)}$. In this regime, $W_1, W_2$ are independent half-normally distributed random variables with parameter $\sigma^2=2$. The asymptotic numerator in \eqref{eq:x_star_body} is thus distributed as $B \sim \frac{1}{2}|W_1 - W_2|$. It is not too difficult to compute the PDF and CDF of the random variable $B$, that are given by
\begin{equation}\label{eq:f_B}
    f_B(x) \coloneqq 2\sqrt{\frac{2}{\pi}}\exp\p{-\frac{x^2}{2}}\erfc\p{\frac{x}{\sqrt{2}}}
\end{equation}
and
\[
    F_B(x) \coloneqq 1-\erfc\p{\frac{x}{\sqrt{2}}}^2,
\]
respectively, where $\text{erfc}(x) \coloneqq \frac{2}{\sqrt{\pi}}\int_x^\infty e^{-t^2}\,dt$ denotes the complementary error function.

A visualization of this PDF against the empirical values from our simulation is provided in \figref{fig:histogram}. According to the derived CDF, we have $\Pr[B > 1] \approx 0.10069$, indicating that more than 10\% of all successful GD runs converge to an equilibrium where the numerator is at least $1$, which is also consistent with our empirical findings. In such instances, achieving a margin that is suboptimal by only $\varepsilon=0.01$ would require a denominator magnitude of approximately $100$, which, given the $\Theta(\ln(t))$ growth rate, necessitates a prohibitively large number of iterations.

\begin{figure}[ht]
    \centering
    \begin{tikzpicture}
        
        \begin{axis}[
            width=10cm,
            height=7cm,
            ymin=0, ymax=400,
            xmin=0, xmax=3.0,
            xlabel={},
            axis lines=left,
            grid=both,
            grid style={dashed, gray!30},
            legend pos=north east,
            bar width=0.1,
            ybar interval,
            xtick={0,0.5,1,1.5,2,2.5,3}
        ]
            \addplot[fill=blue!30, draw=blue!70] coordinates {
                (0.0, 352) (0.1, 344) (0.2, 312) (0.3, 235)
                (0.4, 205) (0.5, 201) (0.6, 145) (0.7, 132)
                (0.8, 114) (0.9, 89) (1.0, 61) (1.1, 41)
                (1.2, 40) (1.3, 25) (1.4, 27) (1.5, 15)
                (1.6, 8) (1.7, 11) (1.8, 3) (1.9, 1)
                (2.0, 3) (2.1, 2) (2.2, 0) (2.3, 1)
                (2.4, 0) (2.5, 0) (2.6, 0) (2.7, 0)
                (2.8, 1) (2.9, 0) (3.0, 0)
            };
            \addlegendentry{$\min\set{|b_2^{(t)} - b_1^{(t)}|,\tfrac12|w_1^{(0)}+w_2^{(0)}|}$}

            \addplot[
                red,
                sharp plot,
                line width=1.5pt,
                domain=0:3,
                samples=200
            ] {236.8*2*sqrt(2/pi) * exp(-x^2/2) * erfcapprox(x/sqrt(2))};
            \addlegendentry{$f_B(x)$}
            
        \end{axis}
    \end{tikzpicture}
    \caption{Empirical distribution of the lower-bound guarantees obtained in the experiment. Among 10{,}000 GD initializations, 2{,}368 converged to a small training loss. Since our simulation conditions all satisfy the monotonicity criterion we establish in \lemref{lem:monotone_limit}, this ensures monotone convergence of the absolute value of the numerator in \eqref{eq:x_star_body} to $\tfrac12|w_1^{(0)}+w_2^{(0)}|$. This implies the lower bound $\displaystyle \min\!\left\{|b_2^{(t)}-b_1^{(t)}|,\;\tfrac12 |w_1^{(0)}+w_2^{(0)}|\right\}$. The red curve is the theoretical density (scaled to match histogram counts) of $\tfrac12|w_1^{(0)}+w_2^{(0)}|$, given in \eqref{eq:f_B}. The close match suggests that this is indeed the distribution of the numerator, not just asymptotically, but one that is also attained in practice.}
    \label{fig:histogram}
\end{figure}
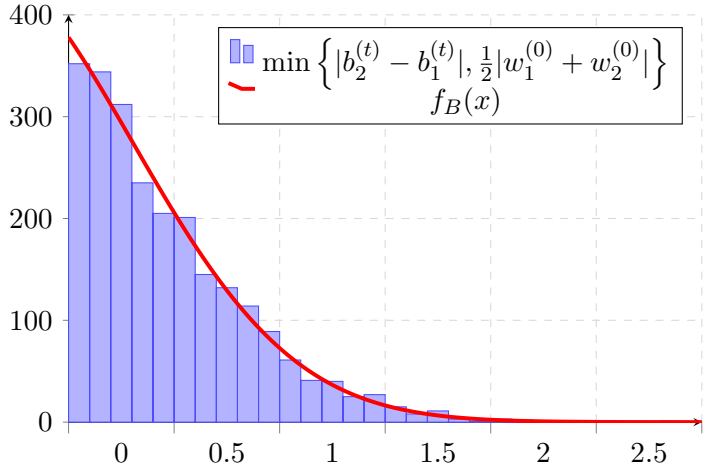

\begin{acknowledgments} 
    Itay Safran is supported by Israel Science Foundation Grant No.\ 1753/25.
\end{acknowledgments}

\newpage
\bibliography{citations}

@article{melamed2023adversarial,
  title={Adversarial examples exist in two-layer relu networks for low dimensional linear subspaces},
  author={Melamed, Odelia and Yehudai, Gilad and Vardi, Gal},
  journal={Advances in Neural Information Processing Systems},
  volume={36},
  pages={5028--5049},
  year={2023}
}

@article{frei2023double,
  title={The double-edged sword of implicit bias: Generalization vs. robustness in relu networks},
  author={Frei, Spencer and Vardi, Gal and Bartlett, Peter and Srebro, Nati},
  journal={Advances in neural information processing systems},
  volume={36},
  pages={8885--8897},
  year={2023}
}

@inproceedings{min2025gradient,
  title={Gradient Flow Provably Learns Robust Classifiers for Orthonormal GMMs},
  author={Min, Hancheng and Vidal, Ren{\'e}},
  booktitle={Forty-second International Conference on Machine Learning},
  year={2025}
}

@inproceedings{ji2021fast,
  title={Fast margin maximization via dual acceleration},
  author={Ji, Ziwei and Srebro, Nathan and Telgarsky, Matus},
  booktitle={International Conference on Machine Learning},
  pages={4860--4869},
  year={2021},
  organization={PMLR}
}

@inproceedings{nacson2019stochastic,
  title={Stochastic gradient descent on separable data: Exact convergence with a fixed learning rate},
  author={Nacson, Mor Shpigel and Srebro, Nathan and Soudry, Daniel},
  booktitle={The 22nd International Conference on Artificial Intelligence and Statistics},
  pages={3051--3059},
  year={2019},
  organization={PMLR}
}

@inproceedings{ji2021characterizing,
  title={Characterizing the implicit bias via a primal-dual analysis},
  author={Ji, Ziwei and Telgarsky, Matus},
  booktitle={Algorithmic Learning Theory},
  pages={772--804},
  year={2021},
  organization={PMLR}
}

@article{ji2018risk,
  title={Risk and parameter convergence of logistic regression},
  author={Ji, Ziwei and Telgarsky, Matus},
  journal={arXiv preprint arXiv:1803.07300},
  year={2018}
}

@inproceedings{nacson2019convergence,
  title={Convergence of gradient descent on separable data},
  author={Nacson, Mor Shpigel and Lee, Jason and Gunasekar, Suriya and Savarese, Pedro Henrique Pamplona and Srebro, Nathan and Soudry, Daniel},
  booktitle={The 22nd International Conference on Artificial Intelligence and Statistics},
  pages={3420--3428},
  year={2019},
  organization={PMLR}
}

@article{soudry2018implicit,
  title={The implicit bias of gradient descent on separable data},
  author={Soudry, Daniel and Hoffer, Elad and Nacson, Mor Shpigel and Gunasekar, Suriya and Srebro, Nathan},
  journal={Journal of Machine Learning Research},
  volume={19},
  number={70},
  pages={1--57},
  year={2018}
}

@inproceedings{wang2024achieving,
  title={Achieving Margin Maximization Exponentially Fast via Progressive Norm Rescaling},
  author={Wang, Mingze and Min, Zeping and Wu, Lei},
  booktitle={International Conference on Machine Learning},
  pages={51124--51160},
  year={2024},
  organization={PMLR}
}

@article{elkabetz2021continuous,
  title={Continuous vs. discrete optimization of deep neural networks},
  author={Elkabetz, Omer and Cohen, Nadav},
  journal={Advances in Neural Information Processing Systems},
  volume={34},
  pages={4947--4960},
  year={2021}
}

@article{mityagin2015zero,
  title={The zero set of a real analytic function},
  author={Mityagin, Boris},
  journal={arXiv preprint arXiv:1512.07276},
  year={2015}
}

@inproceedings{he2015delving,
  title={Delving Deep into Rectifiers: Surpassing Human-Level Performance on ImageNet Classification},
  author={He, Kaiming and Zhang, Xiangyu and Ren, Shaoqing and Sun, Jian},
  booktitle={Proceedings of the IEEE international conference on computer vision (ICCV)},
  pages={1026--1034},
  year={2015}
}

@article{frei2022implicit,
  title={Implicit bias in leaky relu networks trained on high-dimensional data},
  author={Frei, Spencer and Vardi, Gal and Bartlett, Peter L and Srebro, Nathan and Hu, Wei},
  journal={arXiv preprint arXiv:2210.07082},
  year={2022}
}

@article{lyu2021gradient,
  title={Gradient descent on two-layer nets: Margin maximization and simplicity bias},
  author={Lyu, Kaifeng and Li, Zhiyuan and Wang, Runzhe and Arora, Sanjeev},
  journal={Advances in Neural Information Processing Systems},
  volume={34},
  pages={12978--12991},
  year={2021}
}

@inproceedings{sarussi2021towards,
  title={Towards understanding learning in neural networks with linear teachers},
  author={Sarussi, Roei and Brutzkus, Alon and Globerson, Amir},
  booktitle={International Conference on Machine Learning},
  pages={9313--9322},
  year={2021},
  organization={PMLR}
}

@article{ji2020directional,
  title={Directional convergence and alignment in deep learning},
  author={Ji, Ziwei and Telgarsky, Matus},
  journal={Advances in Neural Information Processing Systems},
  volume={33},
  pages={17176--17186},
  year={2020}
}

@article{lyu2019gradient,
  title={Gradient descent maximizes the margin of homogeneous neural networks},
  author={Lyu, Kaifeng and Li, Jian},
  journal={arXiv preprint arXiv:1906.05890},
  year={2019}
}

@inproceedings{biggio2013evasion,
  title={Evasion attacks against machine learning at test time},
  author={Biggio, Battista and Corona, Igino and Maiorca, Davide and Nelson, Blaine and {\v{S}}rndi{\'c}, Nedim and Laskov, Pavel and Giacinto, Giorgio and Roli, Fabio},
  booktitle={Machine Learning and Knowledge Discovery in Databases: European Conference, ECML PKDD 2013, Prague, Czech Republic, September 23-27, 2013, Proceedings, Part III 13},
  pages={387--402},
  year={2013},
  organization={Springer}
}

@article{szegedy2013intriguing,
  title={Intriguing properties of neural networks},
  author={Szegedy, Christian and Zaremba, Wojciech and Sutskever, Ilya and Bruna, Joan and Erhan, Dumitru and Goodfellow, Ian and Fergus, Rob},
  journal={arXiv preprint arXiv:1312.6199},
  year={2013}
}

\newpage

\appendix

\section{Appendix-specific Notation}

We use the shorthand $\btheta^{(t)} \coloneqq (w_1^{(t)}, b_1^{(t)}, w_2^{(t)}, b_2^{(t)})$ to denote the network's weights and biases at step $t$, where $\btheta^{(0)}$ are its parameters at initialization. For brevity, we denote $v_1^{(t)} \coloneqq w_1^{(t)} - b_1^{(t)}$, $v_2^{(t)} \coloneqq w_2^{(t)} - b_2^{(t)}$, $u_1^{(t)} \coloneqq w_1^{(t)} + b_1^{(t)}$ and $u_2^{(t)} \coloneqq w_2^{(t)} + b_2^{(t)}$. We use the shorthands $\beta_j^{(t)}\coloneqq-\frac{b_j^{(t)}}{w_j^{(t)}}$ to denote the breakpoints of the $j^\textrm{th}$ neuron for $j\in\{1,2\}$. We use $x^\star(t)$ to denote the intersection point of the architecture defined in \eqref{eq:architecture} with the $x$-axis (when it is well-defined).

\section{Auxiliary Lemmas}

    The following lemma bounds the updates of GD once a neuron has specialized over a specific instance.
    \begin{lemma}\label{lem:specialization_invariance}
        Consider the objective in \eqref{eq:obj}, and suppose that $w_1^{(t)}>0$ and $\beta_1^{(t)}\in(-1,1)$, and that the GD update at iteration $t$ is $w_1^{(t+1)}=w_1^{(t)} + x$ and $b_1^{(t+1)}=b_1^{(t)} + x$, for some $x>0$. Then, $w_1^{(t+1)}>0$ and $\beta_1^{(t+1)}\in(-1,\beta_1^{(t)})$. 
        
        Likewise, suppose that $w_2^{(t)}<0$ and $\beta_2^{(t)}\in(-1,1)$, and that the GD update at iteration $t$ is $w_2^{(t+1)}=w_2^{(t)} - x$ and $b_1^{(t+1)}=b_1^{(t)} + x$, for some $x>0$. Then, $w_2^{(t+1)}<0$ and $\beta_2^{(t+1)}\in(\beta_2^{(t)},1)$. 
    \end{lemma}

    \begin{proof}
        The invariance of the signs of $w_1^{(t)}$ and $w_2^{(t)}$ follows immediately from the monotonicity of their update rules. We now turn to show the monotonicity of the breakpoints towards $\pm1$. Starting with $\beta_1^{(t)}$, let $w>0$ and $b\in(-w,w)$, which implies $-\frac{b}{w}\in(-1,1)$. Define
        \[
            f(x) \coloneqq -\frac{b+x}{w+x}.
        \]
        To compute the derivative, we rewrite
        \[
            f(x) = -\frac{b - w +x + w}{w+x} = \frac{w-b}{w+x} - 1.
        \]
        Differentiating with respect to $x$ yields
        \[
            f'(x)= \frac{b-w}{(w+x)^2}.
        \]
        Since $b\in(-w,w)$ and $w>0$, the numerator is negative, and therefore $f'(x)<0$ for all $x$, namely $f$ is decreasing. This also implies that $f$ is lower bounded by its limit when $x\to\infty$. We compute $\lim_{x\to\infty}f(x) = -1$, and we have that
        \[
            \beta_1^{(t+1)} = - \frac{b_1^{(t+1)}}{w_1^{(t+1)}} = - \frac{b_1^{(t)} + x}{w_1^{(t)} + x} \in \p{-1,\beta_1^{(t)}},
        \]
        where in the second equality uses the update rules for the weights. Likewise, a similar analysis shows that $\beta_2^{(t+1)}\in\p{\beta_2^{(t)},1}$, concluding the proof of the lemma.
    \end{proof}

    The following lemma allows us to exclude many configurations of the architecture in \eqref{eq:architecture}, by showing that they imply a large training error in \eqref{eq:obj}.
    \begin{lemma}\label{lem:loss_at_least_half}
        Consider the objective in \eqref{eq:obj}, and suppose that either one of the following conditions hold for $\btheta=(w_1,b_1,w_2,b_2)$:
        \begin{itemize}
            \item
            Neuron 1 is right-active ($w_1>0$) and $\beta_1\in[1,\infty)$, or neuron 2 is left-active ($w_2<0$) and $\beta_2\in(-\infty,-1]$,
            \item
            Neuron 1 is left-active and $\beta_1\in(-\infty,1]$, or neuron 2 is right-active and $\beta_2\in[-1,\infty)$,
            \item 
            $w_1\le0$ and $w_2\ge0$,
            \item 
            $w_1=0$ and $b_1\le0$, or $w_2=0$ and $b_2\le0$.
        \end{itemize}
        Then, we have that $\Lcal(\btheta)\ge0.5$.
    \end{lemma}

    \begin{proof}~
        \begin{itemize}
        \item
        First consider all cases where $w_1\le0$ and $w_2\ge0$. By these assumptions, we have that $\relu{w_1x+b_1}$ and $- \relu{w_2x+b_2}$ are both non-increasing, which implies that $\Phi(\btheta;-1)\ge \Phi(\btheta;1)$. Namely, for $y_1\Phi(\btheta;-1)>0$ to hold, we must have $\Phi(\btheta;-1)<0$, which implies $\Phi(\btheta;1)<0$ and $\Lcal(\btheta)\ge0.5$ in either case.
    
        \item 
        Consider all remaining cases in which either $w_1=0$ or $w_2=0$. 
            If $b_1\le0$ or $b_2\le0$, then at least one neuron computes the zero function, and our network reduces to a single neuron. Since such a neuron is always non-positive or non-negative (depending on the sign of $v_j$), at least one instance is misclassified and $\Lcal(\btheta)\ge0.5$.
         \item 
        Consider the remaining cases in which $w_1<0$ and $w_2<0$. 
        \begin{itemize}
            \item
            Suppose that $\beta_1\le1$. Since both neurons are left-active, we have that $\Phi(\btheta;x)=0$ for all sufficiently large $x$. Since $\beta_1\le1$ and $v_2=-1$, we have that $\Phi(\btheta;x)\le0$ for all $x\le1$, which misclassifies $x_2$, and therefore $\Lcal(\btheta)\ge0.5$.
            \item
            If $\beta_1>1$ and $\beta_2\le-1$, then $\Phi(\btheta;x_i)>0$ for all $i\in\{1,2\}$, which misclassifies $x_1=-1$, and therefore $\Lcal(\btheta)\ge0.5$.
        \end{itemize}
        \item 
        Likewise, consider the remaining cases in which $w_1>0$ and $w_2>0$.
        \begin{itemize}
            \item
            Suppose that $\beta_2\ge-1$. Since both neurons are right-active, we have that $\Phi(\btheta;x)=0$ for all sufficiently small $x$. Since $\beta_2\ge-1$ and $v_1=1$, we have that $\Phi(\btheta;x)\ge0$ for all $x\ge-1$, which misclassifies $x_1$, and therefore $\Lcal(\btheta)\ge0.5$.
            \item 
            If $\beta_2<-1$ and $\beta_1\ge1$, then $\Phi(\btheta;x_i)<0$ for all $i\in\{1,2\}$, which misclassifies $x_2=1$, and therefore $\Lcal(\btheta)\ge0.5$.
        \end{itemize}
    \end{itemize}
    To conclude the derivation so far, we have excluded all cases in which $w_1\le0$ or $w_2\ge0$, except for 6 ($w_1,w_2<0$, $\beta_1>1$, $\beta_2>-1$ and $w_1,w_2>0$, $\beta_2<-1$, $\beta_1<1$), in which the loss can be smaller than $0.5$. In all remaining cases, since the breakpoints $\beta_1,\beta_2$ are well-defined, we will split our analysis according to their values.
    
    Moving on to the cases where $w_1>0$ and $w_2<0$:
    \begin{itemize}
        \item 
        If $\beta_1\ge1$, then $\Phi(\btheta;x)\le0$ for all $x\le1$, hence $x_2=1$ is misclassified, and therefore $\Lcal(\btheta)\ge0.5$.
        \item 
        Likewise, if $\beta_2\le-1$, then $\Phi(\btheta;x)\ge0$ for all $x\ge-1$, hence $x_1=-1$ is misclassified, and therefore $\Lcal(\btheta)\ge0.5$.
    \end{itemize}
    We have exhausted all possibilities stated in the lemma, and therefore the proof is concluded.
    \end{proof}

    The following technical lemma allows us to bound the growth rate of recursions that are repeatedly encountered in our analysis.
    \begin{lemma}\label{lem:exp_recursion}
        Suppose that $a_t$ is a sequence such that $a_{t+1}=a_t+C_1\exp(-C_2a_t)$ for some $C_1\neq0$ and $C_2>0$. Then, if $C_1>0$, we have
        \[
            C_2^{-1}\ln\p{\exp(C_2a_0) + tC_2C_1} \le a_{t} \le C_2^{-1}\ln\p{\exp(C_2a_0) + tC_2C_1\exp\p{C_2C_1\exp(-C_2a_0)}},
        \]
        and if $C_1<0$, we have
        \[
            -C_2^{-1}\ln\p{\exp(C_2a_0) - tC_2C_1\exp\p{-C_2C_1\exp(-C_2a_0)}} \le a_{t} \le -C_2^{-1}\ln\p{\exp(C_2a_0) - tC_2C_1}.
        \]
    \end{lemma}

    \begin{proof}
        First, assume that $C_1>0$, and denote $\alpha_t\coloneqq\exp(C_2a_t)$. By the recurrence relation, we have
        \begin{align*}
            \alpha_{t+1} &= \exp(C_2a_{t+1})= \exp\p{C_2a_t+C_2C_1\exp(-C_2a_t)} \\ &= \alpha_t \exp\p{C_2C_1\exp(-C_2a_t)} = \alpha_t\exp\p{\frac{C_2C_1}{\alpha_t}}.
        \end{align*}
        By a standard Taylor expansion argument, we have for all $x\in\reals$ that $1+x\le\exp(x)\le 1+x\exp(x)$. Using these inequalities and the above displayed equation, we arrive at
        \[
            \alpha_t + C_2C_1\le\alpha_{t+1} \le \alpha_t + C_2C_1\exp\p{\frac{C_2C_1}{\alpha_t}}.
        \]
        Since $a_t$ is increasing due to $C_1\exp(-C_2a_t)>0$, we have $\alpha_0\le\alpha_t$ for all $t\ge0$, and therefore
        \[
            C_2C_1\le\alpha_{t+1} - \alpha_{t} \le C_2C_1\exp\p{\frac{C_2C_1}{\alpha_0}}.
        \]
        Summing the above for all $t$, the middle expression telescopes and we obtain
        \[
            tC_2C_1\le\alpha_{t} - \alpha_{0} \le tC_2C_1\exp\p{\frac{C_2C_1}{\alpha_0}}.
        \]
        Adding $\alpha_0$ and taking the logarithm yields
        \[
            C_2^{-1}\ln\p{\alpha_0+tC_2C_1} \le a_{t} \le C_2^{-1}\ln\p{\alpha_0+tC_2C_1\exp\p{\frac{C_2C_1}{\alpha_0}}},
        \]
        and the claim follows by substituting $\alpha_0$ with its definition.

        For the case $C_1<0$, consider the sequence $a'_{t}=-a_t$. Then we have
        \[
            a'_{t+1} = a'_t-C_1\exp(-C_2a_t) = a'_t + |C_1|\exp(-C_2a_t),
        \]
        and the lemma follows by applying the same bounds, substituting $a'_t$ with $a_t$ and $|C_1|$ with $-C_1$, and rearranging.
    \end{proof}

\section{Proof of \thmref{thm:KKT}}\label{app:KKT_proof}

\begin{proof}
    The proof of the theorem relies on a careful case-by-case analysis, where we will exhaustively go over all possible activations patterns of the neurons and examine them for possible solutions. To do this more cleanly, we will derive new equations that are implied by computing the subgradients and substituting them in \eqref{eq:theta}. To this end, we begin with computing the partial derivatives.
    
    We denote the subgradient of the $j^\textrm{th}$ neuron on the $i^\textrm{th}$ instance by 
    \[
        g_{i,j}\coloneqq \begin{cases}
            0, & \text{if } w_jx_i+b_j < 0,\\
            [0,1], & \text{if } w_jx_i+b_j = 0,\\
            1, & \text{if } w_jx_i+b_j > 0.
        \end{cases}
    \]
    For $j\in\{1,2\}$, the partial derivatives are given by
    \begin{align*}
        &\frac{\partial}{\partial w_j}\Phi(\btheta;x) = v_j x g_{i,j},\qquad
        \frac{\partial}{\partial b_j}\Phi(\btheta;x) = v_jg_{i,j}.
    \end{align*}
    Combining these with \eqref{eq:theta}, we obtain for all $j$
    \begin{align}
        w_j \;&=\; v_j \sum_{i=1}^n \lambda_i y_i  x_i g_{i,j}, \label{eq:w_j}\\
        b_j \;&=\; v_j \sum_{i=1}^n \lambda_i y_i  g_{i,j}. \label{eq:b_j}
    \end{align}
    Since the above equations are implied whenever a certain point $\btheta$ satisfies the KKT conditions, satisfying them is a necessary condition for satisying KKT. The converse, however, fails to hold as readily seen by examining the trivial all-zero solution. This solution clearly satisfies Equations~(\ref{eq:w_j},\ref{eq:b_j}), but fails to satisfy \eqref{eq:correct} since the network computes the zero function under this assignment of weights. In light of this, our strategy will be to exhaustively find all feasible solutions to Equations~(\ref{eq:w_j},\ref{eq:b_j}), and subsequently verify their correctness in Equations~(\ref{eq:theta}--\ref{eq:zero_lam}).
    
    Since we have two data instances, each neuron's breakpoint can reside in five different intervals with distinct activation pattern and subgradient behavior: $(-\infty,-1),-1,(-1,1),1,(1,\infty)$. Moreover, each neuron can be right-active ($w_j>0$), left-active ($w_j<0$), or constant with no breakpoint ($w=0$). This amounts to 11 distinct patterns for each neuron, and 121 cases to analyze in total. However, many cases can be excluded immediately by considering the training loss attained on the sample. To this end, we invoke \lemref{lem:loss_at_least_half}, which guarantees that when its conditions hold then the KKT conditions cannot be satisfied. This follows immediately from \eqref{eq:correct}, which implies that the training loss is strictly smaller than $0.5$.
    
    In the remainder of the proof, we turn to examine a much smaller number of cases, that are more complex and require a more careful analysis. This part will also include the only feasible solution for our system.
    \begin{itemize}
        \item 
        We begin with ruling out all $6$ cases in which either $w_1=0$ or $w_2=0$. The only such remaining cases are when $b_1>0$ and $b_2>0$. Let $j\in\{1,2\}$ be the index satisfying $w_j=0$. Then by \eqref{eq:w_j}, we have $0=v_j(\lambda_1+\lambda_2)$. Multiplying this by $v_j$ yields $\lambda_1+\lambda_2=0$. \eqref{eq:primal_feasibility} then implies that $\lambda_1=\lambda_2=0$, which in turn implies that $\btheta=0$ by \eqref{eq:theta}, but then $\Phi(\btheta,\cdot)$ is the zero function, violating \eqref{eq:correct}.
        
        \item
        We are now left with $6+9=15$ cases. Suppose that $w_1<0$, $w_2<0$ and $\beta_1>1$. We have 3 remaining cases under this assumption, determined by the value of $\beta_2$. However, it suffices to consider only the first neuron to show that these cases are infeasible, as we show next.
        
        Since $w_1<0$ and active on both neurons, we get
        \begin{align*}
            w_1 \;&=\;  \sum_{i=1}^n \lambda_i y_i  x_i g_{i,j}=\lambda_1+\lambda_2, \\
            b_1 \;&=\;  \sum_{i=1}^n \lambda_i y_i  g_{i,j}= -\lambda_1 + \lambda_2. 
        \end{align*}
        The assumption in this case requires strict activity on both points, and in particular
            \[
                w_1x_1+b_1>0.
            \]
            Substituting the expressions for $w_1,b_1$ in the above displayed equation yields
            \begin{equation}\label{eq:positive_both_x1}
                w_1x_1+b_1
                = -(\lambda_1+\lambda_2)+(-\lambda_1 + \lambda_2)
                =-2\lambda_1>0,
            \end{equation}
            which implies $\lambda_1<0$, contradicting \eqref{eq:primal_feasibility}.
            Therefore these cases are infeasible.
        
        \item
        Suppose that $w_1>0$ and $w_2>0$. We have 3 remaining cases under this assumption, determined by the value of $\beta_1$. However, it suffices to consider only the second neuron to show that these cases are infeasible, as we show next.
            
        Since $w_2>0$ and active on both neurons we get
        \begin{align*}
            &w_2 \;=\;  -\sum_{i=1}^n \lambda_i y_i  x_i g_{i,j}=-\lambda_1-\lambda_2, \\
            &b_2 \;=\;  -\sum_{i=1}^n \lambda_i y_i  g_{i,j}= \lambda_1 - \lambda_2 . 
        \end{align*}
         The assumption in this case requires strict activity on both points, and in particular
            \[
                w_2x_2+b_2>0.
            \]
            Substituting the expressions for $w_2,b_2$ in the above displayed equation yields
            \begin{equation}\label{eq:positive_both_x2}
                w_2x_2+b_2
                = -\lambda_1-\lambda_2+\lambda_1 - \lambda_2
                =-2\lambda_2>0,
            \end{equation}
            which implies $\lambda_2<0$, contradicting \eqref{eq:primal_feasibility}.
            Therefore these cases are infeasible.
        \item 
        Suppose that $w_1>0$ and $w_2<0$. We have 9 remaining cases under this assumption.
        \begin{itemize}
            \item
            Suppose that $\beta_1<-1$; namely, the first neuron is strictly active on both instances. Despite the opposite orientation of the first neuron, the analysis in this case is identical to the one done in a previous case (see the contradiction derived in \eqref{eq:positive_both_x1}), and it is thus skipped.
    
            \item
            Likewise, the analysis in the case $\beta_2>1$ is identical to the one done in a previous case (see the contradiction derived in \eqref{eq:positive_both_x2}), and therefore it is also skipped.
            
            \item
            There are only 4 cases remaining, which include the case with the only feasible solution. We will now rule out two cases.
            
            Suppose that $\beta_1\in[-1,1)$ and $\beta_2\in(-1,1)$. Focusing on the second neuron, we have by our assumption that $|-\frac{b_2}{w_2}|<1$. Since $w_2$ is left-active, it must hold that $w_2x_1+b_2>0$ and $w_2x_2+b_2<0$, implying that $g_{1,2}=1$ and $g_{2,2}=0$. Substituting these subgradients in Equations~(\ref{eq:w_j},\ref{eq:b_j}) yields
            \[
                w_2=-\lambda_1 g_{1,2}-\lambda_2 g_{2,2}
                =-\lambda_1,
                \qquad
                b_2=\lambda_1 g_{1,2}-\lambda_2 g_{2,2}
                =\lambda_1.
            \]
            
            In particular, it must hold that $b_2=-w_2$ and therefore $\beta_2=1\notin(-1,1)$, contradicting our assumption in this case, hence this activation configuration is infeasible.
            
            \item 

            Suppose that $\beta_1\in(-1,1)$ and $\beta_2=1$. Focusing on the first neuron, we have by our assumption that $|-\frac{b_1}{w_1}|<1$. Since $w_1$ is right-active, it must hold that $w_1x_1+b_1<0$ and $w_1x_2+b_1>0$, implying that $g_{1,1}=0$ and $g_{2,1}=1$. Substituting these subgradients in Equations~(\ref{eq:w_j},\ref{eq:b_j}) yields
            \[
                w_1=\lambda_1 g_{1,1}+\lambda_2 g_{2,1}
                =\lambda_2,
                \qquad
                b_1=-\lambda_1 g_{1,1} + \lambda_2 g_{2,1}
                =\lambda_2.
            \]
            
            In particular, it must hold that $b_1=w_1$ and therefore $\beta_1=-1\notin(-1,1)$, contradicting our assumption in this case, hence this activation configuration is infeasible.
            
            \item
            We are now left with the final case, which contains the only feasible solution. Suppose that $\beta_1=-1$ and $\beta_2=1$. In this case, our assumptions imply strong constraints on the weights. In particular, $\beta_1=-1$ implies $b_1=w_1$; and $\beta_2=1$ implies $b_2=-w_2$. Moreover, since the first neuron is active on $x_2$ and the second is active on $x_1$, we have $g_{1,2}=g_{2,1}=1$; and since both neurons have a subgradient on the remaining instance we have that $g_{1,1}$ and $g_{2,2}$ can take any value in $[0,1]$.
            
            Applying the above to Equations~(\ref{eq:w_j},\ref{eq:b_j}) yields
            \begin{align}
                w_1 &= \lambda_1 g_{1,1}+\lambda_2 g_{2,1}= \lambda_1 g_{1,1}+\lambda_2, \label{eq:KKT_w_1}\\
                w_2 &= -\lambda_1 g_{1,2}-\lambda_2 g_{2,2}=-\lambda_1 -\lambda_2 g_{2,2},\\
                b_1 &= -\lambda_1 g_{1,1}+\lambda_2 g_{2,1}=-\lambda_1 g_{1,1}+\lambda_2,\\
                b_2 &= \lambda_1 g_{1,2}-\lambda_2 g_{2,2}=\lambda_1-\lambda_2 g_{2,2}.\label{eq:KKT_b_2}
            \end{align}
            Plugging the above into the four activation constraints yields
            \begin{align*}
                0 &= w_1x_1+b_1 = -w_1+b_1
                 = -(\lambda_1 g_{1,1}+\lambda_2)
                   +(-\lambda_1 g_{1,1}+\lambda_2)
                 = -2\lambda_1 g_{1,1},\\
                0 &< w_1x_2+b_1 = w_1+b_1
                 = (\lambda_1 g_{1,1}+\lambda_2)
                   +(-\lambda_1 g_{1,1}+\lambda_2)
                 = 2\lambda_2,\\
                0 &< w_2x_1+b_2 = -w_2+b_2
                 = -(-\lambda_1-\lambda_2 g_{2,2})
                   +(\lambda_1-\lambda_2 g_{2,2})
                 = 2\lambda_1,\\
                0 &= w_2x_2+b_2 = w_2+b_2
                 = (-\lambda_1-\lambda_2 g_{2,2})
                   +(\lambda_1-\lambda_2 g_{2,2})
                 = -2\lambda_2 g_{2,2}.
            \end{align*}
            In particular, the middle equations imply $\lambda_1,\lambda_2>0$, which together with the first and last equations forces $g_{1,1}=g_{2,2}=0$. Next, by the middle equations and \eqref{eq:zero_lam}, it must hold that $1=w_1x_2+b_1=2\lambda_2$ and $1=w_2x_1+b_2=2\lambda_1$, implying $\lambda_1=\lambda_2=0.5$. Finally, plugging the lambdas and the subgradients in Equations~(\ref{eq:KKT_w_1}--\ref{eq:KKT_b_2}), we obtain the only feasible solution $\btheta=(w_1,b_1,w_2,b_2)=(0.5,0.5,-0.5,0.5)$. Following our derivation, it is straightforward to plug this solution in Equations~(\ref{eq:theta}--\ref{eq:zero_lam}) and verify that all constraints are indeed satisfied.
        \end{itemize}
    \end{itemize}
    Since Equations~(\ref{eq:theta}--\ref{eq:zero_lam}) assume a normalized margin, we have that the set of all feasible KKT points is given precisely by
    \[
        \set{\p{\alpha,\alpha,-\alpha,\alpha}:\alpha>0}.
    \]
    
    We can finally prove the theorem. Suppose that GF reaches a time $t_0$ such that $\Lcal(\btheta(t_0))<0.5$. By \thmref{thm:known_KKT}, converges in direction to a KKT point. Fix any arbitrary KKT point with parameter $\alpha>0$ that GF converges to, then we have
    \[
        \lim_{t\to\infty}\frac{\btheta(t)}{\norm{\btheta(t)}} = \frac{\p{\alpha,\alpha,-\alpha,\alpha}}{\norm{\p{\alpha,\alpha,-\alpha,\alpha}}} = (0.5,0.5,-0.5,0.5).
    \]
\end{proof}

\section{Proofs for \secref{sec:GD_dynamics}: The Convergence Dynamics of GD}

In this section, we analyze the optimization dynamics of GD on the objective in \eqref{eq:obj}. The section is split into different subsections, depending on the activation patterns studied.

\subsection{Both Neurons Have Specialized}

In this subsection of the appendix, we analyze the gradient descent dynamics after the neurons have specialized. Namely, we consider the regime in which the first neuron is active only on $x_1$, while the second neuron is active only on $x_2$.

We further divide this regime into two sub-cases:
(i) there exists an interval on which both neurons are simultaneously active, and
(ii) there exists an interval on which both neurons are simultaneously inactive.
The following two assumptions formalize these sub-cases.

\begin{assumption}[Both neurons are active on an interval]\label{asm:base_case}
    Given the training set $(x_1,y_1)=(-1,-1)$, $(x_2,y_2)=(1,1)$, consider the network $\Phi(\btheta^{(0)};x)=\relu{w_1^{(0)}x+b_1^{(0)}}-\relu{w_2^{(0)}x+b_2^{(0)}}$. We assume that at initialization we have $w_1^{(0)} > 0$, $w_2^{(0)} < 0$ and $-1 < \beta_1^{(0)} < \beta_2^{(0)} < 1$.
\end{assumption}

\begin{assumption}[Both neurons are inactive on an interval]\label{asm:neurons_active_only_on_one_point}
    Given the training set $(x_1,y_1)=(-1,-1)$, $(x_2,y_2)=(1,1)$, consider the network $\Phi(\btheta^{(0)};x)=\relu{w_1^{(0)}x+b_1^{(0)}}-\relu{w_2^{(0)}x+b_2^{(0)}}$. We assume that at initialization we have $w_1^{(0)} > 0$, $w_2^{(0)} < 0$ and $-1 < \beta_2^{(0)} \le \beta_1^{(0)} < 1$.
\end{assumption}

In section \ref{sec:base_case} we analyze the dynamics of the network under Assumption~\ref{asm:base_case} and in section \ref{sec:neurons_active_only_on_one_point} we analyze the dynamics of the network under Assumption~\ref{asm:neurons_active_only_on_one_point}, but first,
we begin with a few technical lemmas that control the parameters' dynamics.
The following proposition establishes that once each neuron `specializes' on its target training instance, the training dynamics become invariant: the activation pattern and the signs of the weights of the neurons remain fixed.

\begin{proposition}\label{prop:base_case_invariance}
    Under Assumption~\ref{asm:base_case} or \asmref{asm:neurons_active_only_on_one_point}, consider the optimization dynamics of GD on the exponential loss. Then, for any step size $\eta>0$, we have that
    \begin{enumerate}
        \item $\big\{w_1^{(t)}\big\}_{t=0}^{\infty}$ is strictly monotonically increasing,
        \item $\big\{b_1^{(t)}\big\}_{t=0}^{\infty}$ is strictly monotonically increasing,
        \item $\big\{w_2^{(t)}\big\}_{t=0}^{\infty}$ is strictly monotonically decreasing,
        \item $\big\{b_2^{(t)}\big\}_{t=0}^{\infty}$ is strictly monotonically increasing.
    \end{enumerate}
    In particular, the signs of the weights of both neurons remain constant throughout optimization. Moreover, for all $t$, we have
    \begin{enumerate}
        \item
        $\beta_1^{(t+1)}\in\p{-1,\beta_1^{(t)}}$,
        \item
        $\beta_2^{(t+1)}\in\p{\beta_2^{(t)},1}$.
    \end{enumerate}
\end{proposition}
\begin{proof}
    We begin with evaluating $\Phi(\btheta^{(t)}; x)$ at $x_1$ and at $x_2$.
    \begin{align*}
        \Phi(\btheta^{(t)};x_1) &= \relu{-w_1^{(t)} + b_1^{(t)}} - \relu{-w_2^{(t)} + b_2^{(t)}} = w_2^{(t)} - b_2^{(t)}, \\
        \Phi(\btheta^{(t)}, x_2) &= \relu{w_1^{(t)}+b_1^{(t)}} - \relu{w_2^{(t)} + b_2^{(t)}} = w_1^{(t)}+b_1^{(t)}.
    \end{align*}
    The loss at iteration $t$ is thus given by
    \begin{equation*}
        \Lcal\p{\btheta^{(t)}} = \frac{1}{2}\exp\p{w_2^{(t)}-b_2^{(t)}} + \frac{1}{2}\exp\p{-w_1^{(t)}-b_1^{(t)}}.
    \end{equation*}
    With the above, we compute the partial derivatives as follows
    \begin{align*}
        \frac{\partial\Lcal(\btheta^{(t)})}{\partial w_1^{(t)}} &= -\frac{1}{2}\exp\p{-w_1^{(t)}-b_1^{(t)}}, \\
        \frac{\partial \Lcal(\btheta^{(t)})}{\partial b_1^{(t)}} &= -\frac{1}{2}\exp\p{-w_1^{(t)}-b_1^{(t)}}, \\
        \frac{\partial \Lcal(\btheta^{(t)})}{\partial w_2^{(t)}} &= \frac{1}{2}\exp\p{w_2^{(t)} - b_2^{(t)}}, \\
        \frac{\partial \Lcal(\btheta{(t)})}{\partial b_2^{(t)}} &= -\frac{1}{2}\exp\p{w_2^{(t)} - b_2^{(t)}},
    \end{align*}
    which imply the following gradient descent updates
    \begin{align}
        w_1^{(t+1)} &= w_1^{(t)} + \frac{1}{2}\eta \exp\p{-w_1^{(t)} - b_1^{(t)}}, \label{eq:w_1_simple_dynamics}\\
        b_1^{(t+1)} &= b_1^{(t)} + \frac{1}{2}\eta \exp\p{-w_1^{(t)} - b_1^{(t)}}, \label{eq:b_1_simple_dynamics}\\
        w_2^{(t+1)} &= w_2^{(t)} - \frac{1}{2}\eta \exp\p{w_2^{(t)} - b_2^{(t)}}, \label{eq:w_2_simple_dynamics}\\
        b_2^{(t+1)} &= b_2^{(t)} + \frac{1}{2}\eta \exp\p{w_2^{(t)} - b_2^{(t)}}. \label{eq:b_2_simple_dynamics}
    \end{align}
    Using the above update rules together with \lemref{lem:specialization_invariance}, we obtain the monotonicity of each parameter in the desired direction, as well as  the monotonicity of the breakpoints, concluding the proof of the proposition.
\end{proof}

The following two lemmas bound the growth rate of $w_1^{(t)}$ and $w_2^{(t)}$ when both neurons have specialized.
\begin{lemma}\label{lem:converge_rate_w1}
    Under Assumption~\ref{asm:base_case} or \asmref{asm:neurons_active_only_on_one_point}, for any step size $\eta>0$, define the quantity $C_1 = \exp \left( \eta \exp\p{v_1^{(0)}} \exp\p{-2w_1^{(0)}} \right)$ that depends on the initialization parameters and $\eta$. Suppose that the network is optimized using GD with a fixed step size $\eta$ and the exponential loss. Then, we have that $\lim_{t \rightarrow \infty} w_1^{(t)} = \infty$. Moreover,
    \[
        \frac{1}{2}\ln \left(\exp\p{2w_1^{(0)}} + \eta \exp\p{v_1^{(0)}} \cdot t \right) \leq w_1^{(t)} \leq \frac{1}{2} \ln \left( \exp\p{2w_1^{(0)}} + 2 \eta C_1 \exp\p{v_1^{(0)}} \cdot t \right)
    \]
    for all $t$.
\end{lemma}

\begin{proof}
    Denote $E_1^{(t)} \coloneqq w_1^{(t+1)} - w_1^{(t)} \overset{\eqref{eq:w_1_simple_dynamics}}{=} \eta \frac{1}{2} \exp(-w_1^{(t)} -b_1^{(t)})$.
    We observe that by Equations~(\ref{eq:w_1_simple_dynamics},\ref{eq:b_1_simple_dynamics}),
    \[
        w_1^{(t)} - b_1^{(t)} = w_1^{(t-1)}+E_1^{(t-1)}-b_1^{(t-1)}-E_1^{(t-1)} = w_1^{(t-1)} - b_1^{(t-1)},
    \] 
    hence $w_1^{(t)} - b_1^{(t)}$ is constant and it equals $w_1^{(0)} - b_1^{(0)} = v_1^{(0)}$.    
    We rewrite $E_1^{(t)}$ as 
    \[
        E_1^{(t)} = \frac{1}{2}\eta\exp\p{-w_1^{(t)}-b_1^{(t)}} = \frac{1}{2}\eta \exp\p{-2w_1^{(t)} + v_1^{(0)}} = \frac{1}{2}\eta\exp\p{v_1^{(0)}} \exp\p{-2w_1^{(t)}},
    \]
    yielding 
    \[
        w_1^{(t+1)} = w_1^{(t)} + \frac{1}{2} \eta \exp\p{v_1^{(0)}} \exp\p{-2w_1^{(t)}}.
    \] 
    Using \lemref{lem:exp_recursion} with $C_1 = \frac{1}{2} \eta \exp(v_1^{(0)})$ and $C_2 = \frac{1}{2}$ concludes the proof.
\end{proof}

\begin{lemma}\label{lem:converge_rate_w2}
    Under Assumption~\ref{asm:base_case} or \asmref{asm:neurons_active_only_on_one_point}, for any step size $\eta>0$, define the quantity $C_2 = \exp \left( \eta \exp\p{-u_2^{(0)}} \exp\p{2w_2^{(0)}} \right)$ that depends on the initialization parameters and $\eta$. Suppose that the network is optimized using GD with a fixed step size $\eta$ and the exponential loss. Then, we have that $\lim_{t \rightarrow \infty} w_2^{(t)} = -\infty$. Moreover,
    \[
        -\frac{1}{2} \ln \left( \exp\p{-2w_2^{(0)}} +  2\eta \exp\p{-u_2^{(0)}} C_2 \cdot t\right) \leq w_2^{(t)} \leq  -\frac{1}{2} \ln \left( \exp\p{-2w_2^{(0)}} + \eta \exp\p{-u_2^{(0)}} \cdot t\right),
    \]
    for all $t$.
\end{lemma}

\begin{proof}
    We denote $E_2^{(t)} \coloneqq -w_2^{(t+1)} + w_2^{(t)} = \eta \frac{1}{2} \exp(w_2^{(t)} - b_2^{(t)})$. We observe that by Equations~(\ref{eq:w_2_simple_dynamics},\ref{eq:b_2_simple_dynamics}),
    \[
        -w_2^{(t)} - b_2^{(t)} = -w_2^{(t-1)} + E_2^{(t-1)} - b_2^{(t-1)} - E_2^{(t-1)} = -w_2^{(t-1)} - b_2^{(t-1)},
    \]
    hence $-w_2^{(t)}-b_2^{(t)}$ is a constant and it equals $-w_2^{(0)} - b_2^{(0)} = -u_2^{(0)}$.
    We rewrite $E_2^{(t)}$ as 
    \[
        E_2^{(t)} = \frac{1}{2}\eta \exp\p{w_2^{(t)} - b_2^{(t)}} = \frac{1}{2}\eta \exp\p{2w_2^{(t)} - u_2^{(0)}} = \frac{1}{2}\eta\exp\p{-u_2^{(0)}}\exp\p{2w_2^{(t)}},
    \]
    yielding
    \[
        -w_2^{(t+1)} = -w_2^{(t)} + \frac{1}{2}\eta  \exp\p{-u_2^{(0)}} \exp\p{2w_2^{(t)}}.
    \]
    Using \lemref{lem:exp_recursion} with $C_1 = \frac{1}{2}\eta \exp(-u_2^{(0)})$ and $C_2 = 2$ concludes the proof.
\end{proof}

\subsubsection{Both neurons are active on an interval}\label{sec:base_case}
We analyze the dynamics of the network under Assumption~\ref{asm:base_case}.

The margin is defined as the minimal distance between a training instance and the decision boundary. In this particular subsection of the appendix, the decision boundary is the intersection point of the network with the $x$-axis, so by analyzing the dynamics of the intersection point, we are able to analyze the dynamics of the margin. Recall our shorthand for the intersection point after $t$ GD steps, $x^\star(t)$. Under the assumption that both neurons are active on a non-empty set ($\beta_1^{(0)}<\beta_2^{(0)}$), the intersection point is distinct, and it satisfies the equality
\[
    w_1^{(t)}x^\star(t) + b_1^{(t)} = w_2^{(t)}x^\star(t) + b_2^{(t)}.
\]
Solving for $x^{\star}(t)$, we obtain
\begin{equation}
    x^\star(t) = -\,\frac{b_1^{(t)} - b_2^{(t)}}{w_1^{(t)} - w_2^{(t)}}. \label{eq:x_star}
\end{equation}

We note that under Assumption~\ref{asm:base_case} $x^{\star}(t)$ is well-defined for all $t\ge 0$.

Our key proposition in this subsection of the appendix is the following, which establishes tight asymptotic bounds on the rate of convergence to an optimal margin.
\begin{proposition}\label{prop:slow_margin_convergence}
    Under Assumption~\ref{asm:base_case}, consider the optimization dynamics of GD on the exponential loss. For any step size $\eta>0$, suppose that $v_1^{(t)} + u_2^{(t)} \neq 0$. Then,
    \[
        x^\star(t) = \Theta\p{\frac{1}{\ln(t)}}.
    \]
\end{proposition}
The proof of the proposition relies on a few auxiliary lemmas. In what follows, we first establish these required claims, and then turn to prove the proposition.

The following lemma establishes that as $t$ tends to infinity, the system reaches an equilibrium in which the difference of the biases equals the sum of the weights of the neurons.
\begin{lemma}\label{lem:limit_of_sum}
    Under Assumption~\ref{asm:base_case}, consider the optimization dynamics of GD on the exponential loss with any step size $\eta>0$. Then, 
    \[
        \lim _{t \to \infty} b_2^{(t)} - b_1^{(t)} = \lim _{t \to \infty} w_1^{(t)} + w_2^{(t)} = \frac{v_1^{(0)} + u_2^{(0)}}{2}.
    \]
\end{lemma}
\begin{proof}
    First, observe that by the GD updates in Equations~(\ref{eq:w_1_simple_dynamics}--\ref{eq:b_2_simple_dynamics}), we have that for all $t\ge 0$,
    \begin{equation}\label{eq:stability_invariant}
        w_1^{(t)} - b_1^{(t)} + w_2^{(t)} + b_2^{(t)} = w_1^{(0)} - b_1^{(0)} + w_2^{(0)} + b_2^{(0)} = v_1^{(0)} + u_2^{(0)},
    \end{equation}
    Thus proving any of the limits in the statement of the lemma immediately implies the other. 
    It will therefore suffice to prove the lemma for $w_1^{(t)}+w_2^{(t)}$.
    
    For the sake of brevity, we denote $S^\star \coloneqq \frac{v_1^{(0)} + u_2^{(0)}}{2}$ and $S_t \coloneqq w_1^{(t)} + w_2^{(t)}$. Compute
    \begin{align*}
        S_{t+1} &= w_1^{(t+1)} + w_2^{(t+1)} \\
        &= w_1^{(t)} + \eta \exp(-2w_1^{(t)} + v_1^{(0)}) + w_2^{(t)} - \eta \exp(2w_2^{(t)} - u_2^{(0)}) \\
        &= S_t + \eta \left[ \exp(-2w_1^{(t)} + v_1^{(0)}) - \exp(2w_2^{(t)} - u_2^{(0)}) \right].
    \end{align*}
    Denote $D_{t} \coloneqq w_1^{(t)} - w_2^{(t)} > 0$. We can rewrite $w_1^{(t)} = \frac{1}{2}(S_t + D_t)$ and $w_2^{(t)} = \frac{1}{2}(S_t - D_t)$, and therefore we can also rewrite $S_{t+1} - S_t$ using $v_1^{(0)}$, $u_2^{(0)}$ , $S_t$ and $D_t$ as follows
    \begin{align}
        S_{t+1} - S_t &= \eta \left[ \exp(-2w_1^{(t)} + v_1^{(0)}) - \exp(2w_2^{(t)} - u_2^{(0)}) \right] \nonumber \\
        &= \eta \left[ \exp(-S_t -D_t + v_1^{(0)}) - \exp(S_t - D_t - u_2^{(0)}) \right] \nonumber \\
        &= \eta \exp(-D_t)\left[ \exp(v_1^{(0)} - S_t) - \exp(S_t - u_2^{(0)})\right] \nonumber \\
        &= \eta \exp(-D_t) g(S_t), \label{eq:s_t}
    \end{align}
    where $g(s) \coloneqq \exp(v_1^{(0)}-s) - \exp(s - u_2^{(0)})$.
    We observe that by a simple substitution, we have $g(S^\star) = 0$, and that the derivative of $g$ satisfies $g^\prime(s) = -\exp(v_1^{(0)} - s) - \exp(s - u_2^{(0)}) < 0$. Since $g(s)$ is strictly decreasing and continuous, its root $S^{\star}$ is unique. From $g(S^\star) = 0$ and the monotonicity of $g$, it follows that $g(S) > 0$ for all $S < S^\star$, and $g(S) < 0$ for all $S > S^\star$, and thus the following inequality holds for all $S\neq S^{\star}$:
    \begin{equation}\label{eq:gs_sign}
        (S - S^\star) g(S) < 0.
    \end{equation}
    Define the shorthand $\alpha_t \coloneqq \eta \exp(-D_t)$. Using Lemmas~\ref{lem:converge_rate_w1} and \ref{lem:converge_rate_w2}, and their defined quantities $C_1$ and $C_2$, we can bound the growth rate of $D_t$, to obtain
    \[
        \frac{1}{\sqrt{(K_2 + \tilde{G}_2t)(K_1 + \tilde{G}_1t)}} \leq \exp(-D_t) \leq \frac{1}{\sqrt{(K_2 + G_2t)(K_1 + G_1t)}},
    \]
    where $K_1 = \exp(2w_1^{(0)})$, $K_2=\exp(-2w_2^{(0)})$, $G_1 = \eta \exp(v_1^{(0)})$, $G_2 = \eta \exp(-u_2^{(0)})$, $\tilde{G}_1 = 2 \eta C_1\exp(v_1^{(0)})$ and $\tilde{G}_2 = \eta C_2 \exp(-u_2^{(0)})$. For all $t\ge1$, the above implies
    \[
        \frac{1}{t\sqrt{(K_2 + \tilde{G}_2)(K_1 + \tilde{G}_1)}} \leq \exp(-D_t) \leq \frac{1}{t\sqrt{G_2G_1}},
    \]
    and therefore $\exp(-D_t) = \Theta \left( \frac{1}{t}\right)$, which subsequently implies that
    \begin{align}
        & \sum_{t=0}^\infty \alpha_t = \infty, \label{eq:sum_alpha} \\
        & \sum_{t=0}^\infty \alpha_t^2 < \infty. \label{eq:sum_alpha_square}
    \end{align}

    We proceed to show that $S_t$ is bounded. By Lemmas~\ref{lem:converge_rate_w1} and \ref{lem:converge_rate_w2} again, there exist $c_1,c_2,c_3,c_4>0$ and $c'_1,c'_2,c'_3,c'_4>0$ such that
    \[
        0.5\ln(c_1+c_2t) - 0.5\ln(c_3+c_4t) \le S_t = w_1^{(t)}+w_2^{(t)} \le 0.5\ln(c'_1+c'_2t) - 0.5\ln(c'_3+c'_4t)
    \]
    for all $t\ge0$. Define
    \[
        f(t) \coloneqq \ln(c_1+c_2t) - \ln(c_3+c_4t) = \ln\p{\frac{c_1+c_2t}{c_3+c_4t}}.
    \]
    It suffices to show that the natural logarithm of any such rational function is bounded, since this would imply both a lower and an upper bound on $S_t$ for all $t$. To this end, compute
    \[
        f'(t) = \frac{c_2(c_3 + c_4t) - c_4(c_1 + c_2t)}{(c_3 + c_4t)^2} = \frac{c_2c_3 - c_4c_1}{(c_3 + c_4t)^2}.
    \]
    Since the numerator is constant and the denominator is positive, $f$ is monotone, and its value is therefore bounded between $f(0)=\ln(c_1/c_3)$ and $\lim_{t\to\infty}f(t)$. Compute
    \[
        \lim_{t\to\infty}f(t) = \ln\p{\frac{c_1+c_2t}{c_3+c_4t}} = \ln\p{\frac{c_2}{c_4}}.
    \]
    Thus, $S_t$ is bounded.
    
    
    We move on to show that $S_t$ is convergent. Define $V_t \coloneqq (S_t - S^\star)^2$. Expanding $V_{t+1}$ using \eqref{eq:s_t}, we obtain
    \[
        V_{t+1} = (S_t - S^\star + \alpha_t g(S_t))^2 = V_t + 2 \alpha_t (S_t - S^\star) g(S_t) + \alpha_t^2 g(S_t)^2.
    \]
    Summing from $0$ to $T-1$ yields
    \[
        V_T = V_0 + \sum_{i=0}^{T-1}2\alpha_i(S_i - S^\star)g(S_i) + \sum_{i=0}^{T-1}\alpha_i^2g(S_i)^2 = V_0 + \sum_{i=0}^{T-1}A_i + \sum_{i=0}^{T-1}B_i,
    \]
    where $A_i \coloneqq 2\alpha_i(S_i - S^\star)g(S_i) \leq 0$ by \eqref{eq:gs_sign}, and $B_i \coloneqq \alpha_i^2 g(S_i)^2 \geq 0$.
    Rearranging, we obtain
    \[
        V_T + \sum_{i=0}^{T-1}(-A_i) = V_0 + \sum_{i=0}^{T-1}B_i.
    \]
    Since $S_t$ is bounded and $g(s)$ is continuous, $(g(S_t))^2\le M$ is bounded by some $M>0$ for all $t\ge0$; and because $\sum_{t=0}^{\infty}\alpha_t^2<\infty$ by \eqref{eq:sum_alpha_square}, we deduce that $\sum_{i=0}^\infty B_i < \infty$. Rearranging the above displayed equation again, we have
    \[
        V_T + \sum_{i=0}^{T-1}-A_i \leq V_0 + \sum_{i=0}^\infty B_i < \infty.
    \]
    Since both terms on the left-hand side are non-negative, they must be bounded.
    In particular, this implies that
    \[
        \sum_{i=0}^\infty -A_i < \infty,
    \]
    and by multiplying by $-1$ that
    \[
        \sum_{i=0}^\infty A_i > -\infty.
    \]
    Since
    \[
        V_T = V_0 + \sum_{i=0}^{T-1}A_i + \sum_{i=0}^{T-1}B_i
    \]
    and both $\sum A_i$ and $\sum B_i$ are convergent, then $V_t$ must also converge.
    
    Suppose by contradiction that $\lim V_t > 0$. Then, there exist $\delta > 0$ and $t_0\ge0$, such that for all $t \ge t_0$, $|S_t - S^\star| > \delta$. Combining this with \eqref{eq:gs_sign}, we have that for all $t \ge t_0$,
    \[
        (S_t - S^\star) g(S_t) \leq -c < 0
    \]
    for some constant $c > 0$. Therefore, for all $t \ge t_0$ we obtain the inequality
    \[
        V_{t+1} -V_t = A_t + B_t \le -2c \alpha_t + M \alpha_t^2.
    \]
    Since $\alpha_t \to 0$, there exists $t_1 \ge t_0$ such that $M \alpha_t^2 \le c \alpha_t$ for all $t \ge t_1$. We now have for all $t \ge t_1$ that
    \[
        V_{t+1} - V_t < -c \alpha_t.
    \]
    Summing the above from $t_1$ to $T$, we get
    \[
        V_T < V_{t_1} - c \sum_{i=t_1}^T \alpha_i.
    \]
    By \eqref{eq:sum_alpha}, $\sum \alpha_t = \infty$ is divergent, thus the right-hand side tends to $-\infty$ as $T \to \infty$, implying that for sufficiently large $T_0$ we have $V_{T_0} < 0$, contradicting the definition of $V_t = (S_t - S^\star)^2$.
    
    Overall, we obtained that $\lim_{t\to\infty} V_t = 0$, which implies that 
    \[
        \lim_{t\to\infty} S_t = S^\star,
    \]
    and the lemma follows.
\end{proof}

With the above lemmas, we are ready to prove the proposition.

\begin{proof}[Proof of Proposition~\ref{prop:slow_margin_convergence}]
    By \propref{prop:base_case_invariance}, we obtain that the training dynamics are invariant under \asmref{asm:base_case}. In particular, the GD updates remain the same throughout the entire optimization process.
    Using Lemmas~\ref{lem:converge_rate_w1} and \ref{lem:converge_rate_w2}, we deduce that
    \begin{equation}\label{eq:ws_log_growth}
        w_1^{(t)} - w_2^{(t)} = \Theta\p{\ln(t)} \underset{t\to\infty}{\longrightarrow} \infty.
    \end{equation}
    Next, by \lemref{lem:limit_of_sum} we get that
    \[
        \lim_{t \to \infty}b_2^{(t)} - b_1^{(t)} = \frac{v_1^{(0)} + u_2^{(0)}}{2} \neq 0,
    \]
    which is a constant independent of $t$ by our assumption. 
    Overall, we obtain that
    \begin{equation}\label{eq:intersects_zero}
        \lim_{t \to \infty} x^\star(t) \overset{\text{\eqref{eq:x_star}}}{=} \lim_{t \to \infty} -\frac{b_1^{(t)} - b_2^{(t)}}{w_1^{(t)} - w_2^{(t)}} = \lim_{t \to \infty} \frac12\frac{v_1^{(0)} + u_2^{(0)}}{w_1^{(t)} - w_2^{(t)}} = 0.
    \end{equation}
    We now turn to analyze the convergence rate. Assume without loss of generality that $v_1^{(0)} + u_2^{(0)} > 0$ (the other direction is symmetric). By \lemref{lem:limit_of_sum}, there exists large enough $t_0\ge0$ such that for all $t \ge t_0$, we have
    \[
        \frac{v_1^{(0)} + u_2^{(0)}}{4} \le b_2^{(t)} - b_1^{(t)} \le \frac{3(v_1^{(0)} + u_2^{(0)})}{4}.
    \]
    Plugging the above in \eqref{eq:x_star}, while recalling that $w_1^{(t)} - w_2^{(t)} > 0$ for all $t\ge0$, and using \eqref{eq:ws_log_growth}, we obtain
    \[
        x^\star(t) \ge \frac{v_1^{(0)} + u_2^{(0)}}{4\p{w_1^{(t)} - w_2^{(t)}}} \ge \Omega\p{\frac{1}{\ln(t)}},
    \]
    for all $t \ge T$.

    Combining the above lower bound with \eqref{eq:intersects_zero}, it follows that $x^\star(t)=\Theta(1/\ln(t))$, concluding the proof of the proposition.
\end{proof}

\subsubsection{Both neurons are inactive on an interval}\label{sec:neurons_active_only_on_one_point}
In this subsection, we analyze the dynamics of the network under Assumption~\ref{asm:neurons_active_only_on_one_point}.

\begin{proposition}\label{prop:specialized_dead_interval}
    Under Assumption~\ref{asm:neurons_active_only_on_one_point}, consider the optimization dynamics of GD on the exponential loss with a fixed step size $\eta>0$. Then, there exists a constant $T\ge0$ that depends only on $\btheta^{(0)}$ and $\eta$, such that after at most $T$ iterations, the network satisfies \asmref{asm:base_case}.
\end{proposition}

\begin{proof}
    It will suffice to prove that after a constant number of iterations $T$ we have $\beta_1^{(T)}<0<\beta_2^{(T)}$, since this will immediately imply that we have switched to the conditions in \asmref{asm:base_case}.

    First, by \propref{prop:base_case_invariance}, we have that the next iteration either maintains the conditions in \asmref{asm:neurons_active_only_on_one_point}, or switches us to \asmref{asm:base_case}, so we cannot change to any other setting. We will show that both biases eventually become positive, and therefore due to the signs of their corresponding weights, we have $\beta_1^{(T)}<0<\beta_2^{(T)}$.

    Denote $E_t\coloneqq \frac{1}{2}\eta\exp\p{-w_1^{(t)}-b_1^{(t)}}$. Then by Equations~(\ref{eq:w_1_simple_dynamics},\ref{eq:b_1_simple_dynamics}), we have
    \[
        b_1^{(t+1)} - b_1^{(t)} = E_t = w_1^{(t+1)} - w_1^{(t)}.
    \]
    Summing over the above from $0$ to $t-1$ and rearranging yields
    \[
        b_1^{(t)} = w_1^{(t)} - w_1^{(0)} + b_1^{(0)}.
    \]
    By \lemref{lem:converge_rate_w1}, we have $w_1^{(t)}=\Omega(\ln(t))$, where the Omega notation hides constants that depend only on the initialization and the fixed step size $\eta$. This, along with the above displayed equation, implies that $b_1^{(t)}=\Omega(\ln(t))$. Namely, after a constant number of iterations $t_1$, we have $b_1^{(t_1)}>0$, and therefore $\beta_1^{(t_1)}<0$.

    Likewise, by Equations~(\ref{eq:w_2_simple_dynamics},\ref{eq:b_2_simple_dynamics}), we have
    \[
        b_2^{(t+1)} - b_2^{(t)} = \frac{1}{2}\eta\exp\p{w_2^{(t)}-b_2^{(t)}} = -w_2^{(t+1)} + w_2^{(t)}.
    \]
    Summing over the above from $0$ to $t-1$ and rearranging yields
    \[
        b_2^{(t)} = -w_2^{(t)} + w_2^{(0)} + b_2^{(0)}.
    \]
    By \lemref{lem:converge_rate_w2}, we have $-w_2^{(t)}=\Omega(\ln(t))$, where the Omega notation hides constants that depend only on the initialization and the fixed step size $\eta$. This, along with the above displayed equation, implies that $b_2^{(t)}=\Omega(\ln(t))$. Namely, after a constant number of iterations $t_2$, we have $b_2^{(t_2)}>0$, and therefore $\beta_2^{(t_2)}>0$.

    Taking $T=\max\{t_1,t_2\}$, the proposition follows after at most $T$ iterations.
\end{proof}

\subsection{Only One Neuron Has Specialized}
In this section, we analyze all the cases where one neuron has specialized on its instance, while the other neuron is active on both. We limit ourselves to cases where $\Lcal(\btheta^{(0)}) < 0.5$, since these are the cases of interest in our analysis, in which a perfect classification error is obtained.

\subsubsection{Right pointing neuron active on both, left pointing neuron specialized}\label{app:case_1}
\begin{assumption}\label{asm:one_active_both_one_active_one_case_1}
    Given the training set $(x_1,y_1)=(-1,-1)$, $(x_2,y_2)=(1,1)$, consider the network $\Phi(\btheta^{(0)};x)=\relu{w_1^{(0)}x+b_1^{(0)}}-\relu{w_2^{(0)}x+b_2^{(0)}}$. We assume that at initialization we have $w_1^{(0)} > 0$, $w_2^{(0)} < 0$, $\beta_1^{(0)}<-1$ and $\beta_2^{(0)} \in(-1,1)$, and that $\Lcal(\btheta^{(t)})<0.5$ for all $t\ge0$.
\end{assumption}

\begin{proposition}\label{prop:one_active_both_one_active_one_case_1}
    Under \asmref{asm:one_active_both_one_active_one_case_1}, consider the optimization dynamics of GD on the exponential loss with a fixed step size $\eta>0$. Then, there exists a constant $T\ge0$ that depends only on $\btheta^{(0)}$ and $\eta$, such that after at most $T$ iterations, the network satisfies \asmref{asm:base_case} or \asmref{asm:neurons_active_only_on_one_point}.
\end{proposition}

\begin{proof}
    For a network that satisfies \asmref{asm:one_active_both_one_active_one_case_1}, the outputs of the network on the training points are
    \begin{align*}
        &\Phi(\btheta, x_1) = \relu{-w_1 + b_1} - \relu{-w_2 + b_2} = -w_1 + b_1 +w_2 - b_2, \\
        &\Phi(\btheta; x_2) = \relu{w_1 + b_1} - \relu{w_2 + b_2}=  w_1 + b_1,
    \end{align*}
    and the training loss is therefore
    \begin{equation}\label{eq:loss_two_one}
        \Lcal(\btheta) = \frac12\exp\p{-w_1+b_1+w_2-b_2} + \frac12\exp\p{-w_1 - b_1}.
    \end{equation}
    With the above, we can compute the partial derivatives to obtain
    \begin{align*}
        \frac{\partial \Lcal(\btheta)}{\partial w_1} &= -\frac12\exp(-w_1 + b_1 + w_2 - b_2) - \frac12\exp(-w_1-b_1), \\
        \frac{\partial \Lcal(\btheta)}{\partial b_1} &= \frac12\exp(-w_1 + b_1 + w_2 - b_2) - \frac12\exp(-w_1 - b_1), \\
        \frac{\partial \Lcal(\btheta)}{\partial w_2} &= \frac12\exp(-w_1 + b_1 + w_2 - b_2), \\
        \frac{\partial \Lcal(\btheta)}{\partial b_2} &= -\frac12\exp(-w_1 + b_1 + w_2 - b_2).
    \end{align*}
    Thus, the GD updates are given by
    \begin{align}
        w_1^{(t+1)} &= w_1^{(t)} + \frac12\eta \exp\p{-w_1^{(t)} + b_1^{(t)} + w_2^{(t)} - b_2^{(t)}} + \frac12\eta\exp\p{-w_1^{(t)}-b_1^{(t)}}, \label{eq:w_1_case_1}\\
        b_1^{(t+1)} &= b_1^{(t)} - \frac12\eta \exp\p{-w_1^{(t)} + b_1^{(t)} + w_2^{(t)} - b_2^{(t)}} + \frac12\eta\exp\p{-w_1^{(t)} - b_1^{(t)}}, \label{eq:b_1_case_1}\\
        w_2^{(t+1)} &= w_2^{(t)} - \frac12\eta \exp\p{-w_1^{(t)} + b_1^{(t)} + w_2^{(t)} - b_2^{(t)}}, \label{eq:w_2_case_1}\\
        b_2^{(t+1)} &= b_2^{(t)} + \frac12\eta \exp\p{-w_1^{(t)} + b_1^{(t)} + w_2^{(t)} - b_2^{(t)}}. \label{eq:b_2_case_1}
    \end{align}
    By the above update rules and \lemref{lem:specialization_invariance}, we have that the signs of the weights are invariant, as well as the breakpoint $\beta_2^{(t)}$. Since $\Lcal(\btheta^{(t)})<0.5$ for all $t\ge0$, the first neuron cannot turn inactive on $x_2$, and we thus deduce that we must remain under the conditions of \asmref{asm:one_active_both_one_active_one_case_1}, or transition to \asmref{asm:base_case} or \asmref{asm:neurons_active_only_on_one_point}. We will now show that we cannot remain in \asmref{asm:one_active_both_one_active_one_case_1} for more than a constant number of iterations.
    
    From Equations~(\ref{eq:w_1_case_1},\ref{eq:b_1_case_1}), we get
    \begin{equation}\label{eq:v_1_case_1}
        v_1^{(t+1)} = w_1^{(t+1)} - b_1^{(t+1)} = w_1^{(t)} - b_2^{(t)} + \eta \exp\p{-w_1^{(t)} + b_1^{(t)} + w_2^{(t)} - b_2^{(t)}} = v_1^{(t)} + \eta \exp\p{v_2^{(t)} - v_1^{(t)}},
    \end{equation}
    and from Equations~(\ref{eq:w_2_case_1},\ref{eq:b_2_case_1}) we obtain
    \begin{equation}\label{eq:v_2_case_1}
        v_2^{(t+1)} = w_2^{(t+1)} - b_2^{(t+1)} = w_2^{(t)} - b_2^{(t)} - \eta \exp\p{-w_1^{(t)} + b_1^{(t)} + w_2^{(t)} - b_2^{(t)}} = v_2^{(t)} - \eta \exp\p{v_2^{(t)} - v_1^{(t)}}.
    \end{equation}
    Using the shorthand $S \coloneqq v_1^{(0)} + v_2^{(0)}$, we observe that
    \[
        v_1^{(t+1)} + v_2^{(t+1)} =  v_1^{(t)} + \eta \exp\p{v_2^{(t)} - v_1^{(t)}} + v_2^{(t)} - \eta \exp\p{v_2^{(t)} - v_1^{(t)}} = v_1^{(t)} + v_2^{(t)}.
    \]
    By induction, we get that $v_1^{(t)} + v_2^{(t)} = S$ for all $t$. Substituting $v_2^{(t)} = S - v_1^{(t)}$ into the update rule for $v_1^{(t)}$ in \eqref{eq:v_1_case_1}, we have
    \[
        v_1^{(t+1)} = v_1^{(t)} + \eta \exp\p{S - 2v_1^{(t)}} = v_1^{(t)} + \eta\exp(S) \exp\p{-2v_1^{(t)}}.
    \]
    Applying \lemref{lem:exp_recursion} to the above recursion, we have that $v_1^{(t)}$ increases at a logarithmic rate. In particular, there exists an initialization-dependent constant $T$ such that $v_1^{(T)}>0$, implying $w_1^{(T)}>b_1^{(T)}$. Since $w_1^{(T)}>0$, we divide by $w_1^{(T)}$ and multiply by $-1$ to obtain
    \[
        \beta_1^{(T)}>-1,
    \]
    implying that the first neuron is no longer active on $x_2$.
\end{proof}

\subsubsection{Left pointing neuron active on both, right pointing neuron specialized}\label{app:case_2}
\begin{assumption}\label{asm:one_active_both_one_active_one_case_2}
    Given the training set $(x_1,y_1)=(-1,-1)$, $(x_2,y_2)=(1,1)$, consider the network $\Phi(\btheta^{(0)};x)=\relu{w_1^{(0)}x+b_1^{(0)}}-\relu{w_2^{(0)}x+b_2^{(0)}}$. We assume that at initialization we have $w_1^{(0)} > 0$, $w_2^{(0)} < 0$, $\beta_1^{(0)}\in(-1,1)$ and $\beta_2^{(0)} > 1$, and that $\Lcal(\btheta^{(t)})<0.5$ for all $t\ge0$.
\end{assumption}

\begin{proposition}\label{prop:one_active_both_one_active_one_case_2}
    Under \asmref{asm:one_active_both_one_active_one_case_2}, consider the optimization dynamics of GD on the exponential loss with a fixed step size $\eta>0$. Then, there exists a constant $T\ge0$ that depends only on $\btheta^{(0)}$ and $\eta$, such that after at most $T$ iterations, the network satisfies \asmref{asm:base_case} or \asmref{asm:neurons_active_only_on_one_point}.
\end{proposition}

\begin{proof}
    For a network that satisfies \asmref{asm:one_active_both_one_active_one_case_2}, the values of the network on the training points are
    \begin{align*}
        \Phi(\btheta, x_1) &= \relu{-w_1 + b_1} - \relu{-w_2 + b_2} = w_2 - b_2, \\
        \Phi(\btheta; x_2) &= \relu{w_1 + b_1} - \relu{w_2 + b_2}=  w_1 + b_1 - w_2 - b_2,
    \end{align*}
    and the training loss is therefore
    \[
        \Lcal(\btheta) = \frac12\exp(w_2 - b_2) + \frac12\exp(w_2+b_2-w_1-b_1).
    \]
    With the above, we can compute the partial derivatives to obtain
    \begin{align*}
        \frac{\partial \Lcal(\btheta)}{\partial w_1} &= -\frac12\exp(w_2+b_2-w_1-b_1), \\
        \frac{\partial \Lcal(\btheta)}{\partial b_1} &= -\frac12\exp(w_2+b_2-w_1-b_1), \\
        \frac{\partial \Lcal(\btheta)}{\partial w_2} &= \frac12\exp(w_2-b_2) + \frac12\exp(w_2+b_2-w_1-b_1), \\
        \frac{\partial \Lcal(\btheta)}{\partial b_2} &= -\frac12\exp(w_2 - b_2) + \frac12\exp(w_2+b_2-w_1-b_1).
    \end{align*}
    Thus, the GD updates are given by
    \begin{align}
        w_1^{(t+1)} &= w_1^{(t)} + \frac12\eta\exp\p{w_2^{(t)}+b_2^{(t)}-w_1^{(t)}-b_1^{(t)}}, \label{eq:w_1_case_2}\\
        b_1^{(t+1)} &= b_1^{(t)} + \frac12\eta \exp\p{w_2^{(t)}+b_2^{(t)}-w_1^{(t)}-b_1^{(t)}}, \label{eq:b_1_case_2}\\
        w_2^{(t+1)} &= w_2^{(t)} - \frac12\eta\exp\p{w_2^{(t)}-b_2^{(t)}} - \frac12\eta\exp\p{w_2^{(t)}+b_2^{(t)}-w_1^{(t)}-b_1^{(t)}}, \label{eq:w_2_case_2}\\
        b_2^{(t+1)} &= b_2^{(t)} +\frac12\eta\exp\p{w_2^{(t)} - b_2^{(t)}} - \frac12\eta\exp\p{w_2^{(t)}+b_2^{(t)}-w_1^{(t)}-b_1^{(t)}}. \label{eq:b_2_case_2}
    \end{align}

    By the above update rules and \lemref{lem:specialization_invariance}, we have that the signs of the weights are invariant, as well as the breakpoint $\beta_1^{(t)}$. Since $\Lcal(\btheta^{(t)})<0.5$ for all $t\ge0$, the second neuron cannot turn inactive on $x_1$, and we thus deduce that we must remain under the conditions of \asmref{asm:one_active_both_one_active_one_case_2}, or transition to \asmref{asm:base_case} or \asmref{asm:neurons_active_only_on_one_point}. We will now show that we cannot remain in \asmref{asm:one_active_both_one_active_one_case_2} for more than a constant number of iterations.

    From Equations~(\ref{eq:w_1_case_2},\ref{eq:b_1_case_2}) we get
    \begin{equation}\label{eq:u_1_case_2}
        u_1^{(t+1)} = w_1^{(t+1)} + b_1^{(t+1)} = w_1^{(t)} + b_1^{(t)} + \eta \exp(w_2^{(t)} + b_2^{(t)} - w_1^{(t)} - b_1^{(t)}) = u_1^{(t)} + \eta \exp(u_2^{(t)} - u_1^{(t)}),
    \end{equation}
    and from Equations~(\ref{eq:w_2_case_2},\ref{eq:b_2_case_2}) we obtain
    \begin{equation}\label{eq:u_2_case_2}
        u_2^{(t+1)} = w_2^{(t+1)} + b_2^{(t+1)} = w_2^{(t)} + b_2^{(t)} - \eta \exp(w_2^{(t)} + b_2^{(t)} - w_1^{(t)} - b_1^{(t)}) = u_2^{(t)} - \eta \exp(u_2^{(t)} - u_1^{(t)}).
    \end{equation}
    Using the shorthand $S \coloneqq u_1^{(0)} + u_2^{(0)}$, we observe that
    \[
        u_1^{(t+1)} + u_2^{(t+1)} =  u_1^{(t)} + \eta \exp\p{u_2^{(t)} - u_1^{(t)}} + u_2^{(t)} - \eta \exp\p{u_2^{(t)} - u_1^{(t)}} = u_1^{(t)} + u_2^{(t)}.
    \]
    By induction we get that $u_1^{(t)} + u_2^{(t)} = S$ for all $t$. Substituting $u_1^{(t)} = S - u_2^{(t)}$ into the update rule for $u_2$ in \eqref{eq:u_2_case_2}, we have
    \[
        u_2^{(t+1)} = u_2^{(t)} - \eta \exp\p{2u_2^{(t)} - S} = u_2^{(t)} - \eta\exp(-S) \exp\p{2u_2^{(t)}}.
    \]
    Applying \lemref{lem:exp_recursion} to the above recursion, we have that $u_2^{(t)}$ decreases at a logarithmic rate. In particular, there exists an initialization-dependent constant $T$ such that $u_2^{(T)}<0$, implying $w_2^{(T)}<-b_1^{(T)}$. Since $w_2^{(T)}<0$, we divide by $w_2^{(T)}$ to obtain
    \[
        \beta_2^{(T)}<1,
    \]
    implying that the second neuron is no longer active on $x_1$.
\end{proof}

\subsubsection{Both neurons are right-active}\label{sec:one_active_both_one_active_one_case_3}
\begin{assumption}[Both neurons are right-active]\label{asm:one_active_both_one_active_one_case_3}
    Given the training set $(x_1,y_1)=(-1,-1)$, $(x_2,y_2)=(1,1)$, consider the network $\Phi(\btheta^{(0)};x)=\relu{w_1^{(0)}x+b_1^{(0)}}-\relu{w_2^{(0)}x+b_2^{(0)}}$. We assume that at initialization we have $w_1^{(0)} > 0$, $w_2^{(0)} > 0$, $\beta_1^{(0)}\in(-1,1)$ and $\beta_2^{(0)} < -1$, and that $\Lcal(\btheta^{(t)})<0.5$ for all $t\ge0$.
\end{assumption}

\begin{proposition}\label{prop:one_active_both_one_active_one_case_3}
    Under \asmref{asm:one_active_both_one_active_one_case_3}, consider the optimization dynamics of GD on the exponential loss with a fixed step size $\eta>0$. Then, there exists a constant $T\ge0$ that depends only on $\btheta^{(0)}$ and $\eta$, such that after at most $T$ iterations, the network satisfies either \asmref{asm:base_case}, \asmref{asm:neurons_active_only_on_one_point} or \asmref{asm:one_active_both_one_active_one_case_2}.
\end{proposition}

\begin{proof}
    For a network that satisfies \asmref{asm:one_active_both_one_active_one_case_3}, the training dynamics are identical to those analyzed in a previous case (see \appref{app:case_2}).

    By the update rules in Equations~(\ref{eq:w_1_case_2}--\ref{eq:b_2_case_2}) and by \lemref{lem:specialization_invariance}, we have that the sign of $w_1^{(t)}$ is invariant, as well as the breakpoint $\beta_1^{(t)}$. Since $\Lcal(\btheta^{(t)})<0.5$ for all $t\ge0$, the second neuron cannot turn inactive on $x_1$, and we thus deduce that we must remain under the conditions of \asmref{asm:one_active_both_one_active_one_case_3}, or transition to either one of \asmref{asm:base_case}, \asmref{asm:neurons_active_only_on_one_point} or \asmref{asm:one_active_both_one_active_one_case_2}. We will now show that we cannot remain in \asmref{asm:one_active_both_one_active_one_case_3} for more than a constant number of iterations.

    Recall the shorthand $S \coloneqq u_1^{(0)} + u_2^{(0)}$, we plug $u_2^{(t)}=S-u_1^{(t)}$ in \eqref{eq:u_1_case_2}, to obtain
    \[
        u_1^{(t+1)} = u_1^{(t)} + \eta \exp\p{S - 2u_1^{(t)}} = u_1^{(t)} + \eta\exp(S) \exp\p{-2u_1^{(t)}}.
    \]
    Applying \lemref{lem:exp_recursion} to the above recursion, we have that
    \[
        u_1^{(t)}-u_1^{(0)} = \sum_{i=0}^t \eta\exp(S) \exp\p{-2u_1^{(t)}} \ge C\ln(t),
    \]
    for some initialization-dependent constant $C>0$. Let $T$ be large enough such that $C\ln(T) > 2w_2^{(0)}$, then by \eqref{eq:w_2_case_2} we have
    \begin{align*}
        w_2^{(t)} &= w_2^{(0)} -\frac12\sum_{i=0}^t\eta\exp\p{v_2^{(t)}} -\frac12\sum_{i=0}^t\eta\exp(S) \exp\p{-2u_1^{(t)}} \\
        &\le w_2^{(0)} -\frac12\sum_{i=0}^t\eta\exp(S) \exp\p{-2u_1^{(t)}} < w_2^{(0)} - w_2^{(0)} = 0,
    \end{align*}
    and therefore $w_2^{(0)}$ changes sign after at most $T$ iterations, so we are no longer under \asmref{asm:one_active_both_one_active_one_case_3}.
\end{proof}

\subsubsection{Both neurons are left-active}\label{sec:one_active_both_one_active_one_case_4}
\begin{assumption}[Both neurons are left-active]\label{asm:one_active_both_one_active_one_case_4}
    Given the training set $(x_1,y_1)=(-1,-1)$, $(x_2,y_2)=(1,1)$, consider the network $\Phi(\btheta^{(0)};x)=\relu{w_1^{(0)}x+b_1^{(0)}}-\relu{w_2^{(0)}x+b_2^{(0)}}$. We assume that at initialization we have $w_1^{(0)} < 0$, $w_2^{(0)} < 0$, $\beta_1^{(0)}>1$ and $\beta_2^{(0)} \in (-1,1)$, and that $\Lcal(\btheta^{(t)})<0.5$ for all $t\ge0$.
\end{assumption}

\begin{proposition}\label{prop:one_active_both_one_active_one_case_4}
    Under \asmref{asm:one_active_both_one_active_one_case_4}, consider the optimization dynamics of GD on the exponential loss with a fixed step size $\eta>0$. Then, there exists a constant $T\ge0$ that depends only on $\btheta^{(0)}$ and $\eta$, such that after at most $T$ iterations, the network satisfies either \asmref{asm:base_case}, \asmref{asm:neurons_active_only_on_one_point} or \asmref{asm:one_active_both_one_active_one_case_1}.
\end{proposition}

\begin{proof}
    For a network that satisfies \asmref{asm:one_active_both_one_active_one_case_4}, the training dynamics are identical to those analyzed in a previous case (see \appref{app:case_1}).

    By the update rules in Equations~(\ref{eq:w_1_case_1}--\ref{eq:b_2_case_1}) and by \lemref{lem:specialization_invariance}, we have that the sign of $w_2^{(t)}$ is invariant, as well as the breakpoint $\beta_2^{(t)}$. Since $\Lcal(\btheta^{(t)})<0.5$ for all $t\ge0$, the first neuron cannot turn inactive on $x_2$, and we thus deduce that we must remain under the conditions of \asmref{asm:one_active_both_one_active_one_case_4}, or transition to either one of \asmref{asm:base_case}, \asmref{asm:neurons_active_only_on_one_point} or \asmref{asm:one_active_both_one_active_one_case_1}. We will now show that we cannot remain in \asmref{asm:one_active_both_one_active_one_case_4} for more than a constant number of iterations.

    Recall the shorthand $S \coloneqq v_1^{(0)} + v_2^{(0)}$, we plug $v_1^{(t)}=S-v_2^{(t)}$ in \eqref{eq:v_2_case_1}, to obtain
    \[
        v_2^{(t+1)} = v_2^{(t)} - \eta \exp\p{2v_2^{(t)} - S} = v_2^{(t)} - \eta\exp(-S) \exp\p{2v_2^{(t)}}.
    \]
    Applying \lemref{lem:exp_recursion} to the above recursion, we have that
    \[
        v_2^{(t)}-v_2^{(0)} = -\sum_{i=0}^t \eta\exp(-S) \exp\p{2v_2^{(t)}} \le -C\ln(t),
    \]
    for some initialization-dependent constant $C>0$. Let $T$ be large enough such that $C\ln(T) > -2w_1^{(0)}$, then by \eqref{eq:w_1_case_1} we have
    \begin{align*}
        w_1^{(t)} &= w_1^{(0)} +\frac12\sum_{i=0}^t\eta\exp\p{-u_1^{(t)}} +\frac12\sum_{i=0}^t\eta\exp(-S) \exp\p{2v_2^{(t)}} \\
        &\ge w_1^{(0)} + \frac12\sum_{i=0}^t\eta\exp(-S) \exp\p{2v_2^{(t)}} > w_1^{(0)} - w_1^{(0)} = 0,
    \end{align*}
    and therefore $w_1^{(0)}$ changes sign after at most $T$ iterations, so we are no longer under \asmref{asm:one_active_both_one_active_one_case_4}.
\end{proof}

\subsection{Both neurons are active on both training points}
    \begin{assumption}[Both neurons are both active]\label{asm:linear_setting}
        Given the training set $(x_1,y_1)=(-1,-1)$, $(x_2,y_2)=(1,1)$, suppose that the network $\Phi(\btheta^{(0)};x)=\relu{w_1^{(0)}x+b_1^{(0)}}-\relu{w_2^{(0)}x+b_2^{(0)}}$ satisfies $\Lcal(\btheta^{(0)})<0.5$, and that both neurons are active on both instances.
    \end{assumption}
        
    \begin{proposition}\label{prop:linear_setting}
        Under Assumption~\ref{asm:linear_setting}, consider the optimization dynamics of GD on the exponential loss. Assume that for any step size $\eta>0$, there exists $T>0$ that depends only on $\btheta^{(0)}$ and $\eta$, such that after $T$ iterations of GD, we have that at least one neuron turns inactive on a single instance, and .
    \end{proposition}

    \begin{proof}
        We begin with evaluating $\Phi(\btheta^{(t)}; x)$ at the data points $x_1=-1$ and at $x_2=1$. Since both neurons are active on both instances, we have
        \begin{align*}
            &\Phi\p{\btheta^{(t)};x_1} = \relu{-w_1^{(t)} + b_1^{(t)}} - \relu{-w_2^{(t)} + b_2^{(t)}} = -w_1^{(t)}+b_1^{(t)} + w_2^{(t)} - b_2^{(t)}, \\
            &\Phi\p{\btheta^{(t)};x_2} = \relu{w_1^{(t)} + b_1^{(t)}} - \relu{w_2^{(t)} + b_2^{(t)}} = w_1^{(t)}+b_1^{(t)} - w_2^{(t)} - b_2^{(t)}.
        \end{align*}
        The loss at iteration $t$ is thus given by
        \begin{equation}\label{eq:loss_linear_setting}
            \Lcal\p{\btheta^{(t)}} = \frac{1}{2}\exp\p{-w_1^{(t)}+b_1^{(t)} + w_2^{(t)} - b_2^{(t)}} + \frac{1}{2}\exp\p{-w_1^{(t)}-b_1^{(t)} + w_2^{(t)} + b_2^{(t)}}.
        \end{equation}
        Using the shorthands $A_t\coloneqq\exp\p{-w_1^{(t)}+b_1^{(t)} + w_2^{(t)} - b_2^{(t)}}$ and $B_t\coloneqq\exp\p{-w_1^{(t)}-b_1^{(t)} + w_2^{(t)} + b_2^{(t)}}$, the partial derivatives can be written compactly as
        \begin{align*}
            \frac{\partial\Lcal(\btheta^{(t)})}{\partial w_1^{(t)}} &= -\frac{1}{2}A_t - \frac{1}{2}B_t, \\
            \frac{\partial \Lcal(\btheta^{(t)})}{\partial b_1^{(t)}} &= \frac{1}{2}A_t - \frac{1}{2}B_t, \\
            \frac{\partial \Lcal(\btheta^{(t)})}{\partial w_2^{(t)}} &= \frac{1}{2}A_t + \frac{1}{2}B_t, \\
            \frac{\partial \Lcal(\btheta{(t)})}{\partial b_2^{(t)}} &= - \frac{1}{2}A_t + \frac{1}{2}B_t.
        \end{align*}
        Using the above to compute the gradient step update for each parameter, we get
        \begin{align}
            w_1^{(t+1)} &= w_1^{(t)} + \frac{1}{2}\eta A_t + \frac{1}{2}\eta B_t, \label{eq:w_1_linear_rec}\\
            b_1^{(t+1)} &= b_1^{(t)} - \frac{1}{2}\eta A_t + \frac{1}{2}\eta B_t, \label{eq:b_1_linear_rec}\\
            w_2^{(t+1)} &= w_2^{(t)} - \frac{1}{2}\eta A_t - \frac{1}{2}\eta B_t,\nonumber\\
            b_2^{(t+1)} &= b_2^{(t)} + \frac{1}{2}\eta A_t - \frac{1}{2}\eta B_t.\nonumber
        \end{align}
        We now bound the difference $v_1^{(t)}\coloneqq w_1^{(t)}-b_1^{(t)}$, by subtracting \eqref{eq:b_1_linear_rec} from \eqref{eq:w_1_linear_rec} to obtain
        \begin{equation}\label{eq:v_1_linear_step_size}
            v_1^{(t+1)} = v_1^{(t)} + \eta A_t.
        \end{equation}
        Additionally, summing the gradient updates of the weights and the biases separately,  we get
        \begin{align*}
            w_1^{(t+1)} + w_2^{(t+1)} &= w_1^{(t)} + w_2^{(t)}, \\
            b_1^{(t+1)} + b_2^{(t+1)} &= b_1^{(t)} + b_2^{(t)}.
        \end{align*}
        Recursively applying the above, we arrive at
        \begin{align*}
            w_1^{(t)} + w_2^{(t)} &= w_1^{(0)} + w_2^{(0)},\\
            b_1^{(t)} + b_2^{(t)} &= b_1^{(0)} + b_2^{(0)}.
        \end{align*}
        By adding the right-hand-side and subtracting the left-hand-side of the above, we can simplify $A_t$ as follows
        \begin{align*}
            A_t &= \exp\p{-w_1^{(t)}+b_1^{(t)} + w_2^{(t)} - b_2^{(t)}} = \exp\p{-2w_1^{(t)}+2b_1^{(t)} + w_1^{(0)} - b_1^{(0)} + w_2^{(0)} - b_2^{(0)}} \\
            &=\exp\p{v_1^{(0)}+v_2^{(0)}}\exp\p{-2v_1^{(t)}}.
        \end{align*}
        Plugging the above in \eqref{eq:v_1_linear_step_size}, we obtain the recursion
        \begin{equation}\label{eq:v_1_linear_recursion}
            v_1^{(t+1)} = v_1^{(t)} + \eta\exp\p{v_1^{(0)}+v_2^{(0)}}\exp\p{-2v_1^{(t)}}.
        \end{equation}
        By our assumption that both neurons are active on both instances, we have that $b_1^{(0)}>0$ (since otherwise at least one of $w_1^{(0)}x_1+b_1^{(0)}$ and $w_1^{(0)}x_2+b_1^{(0)}$ is negative). Moreover, we have that $|w_1^{(0)}|<b_1^{(0)}$ (since otherwise the breakpoint of the first neuron which is $-b_1^{(0)}/w_1^{(0)}$ will be in $[-1,1]$). We therefore have that $v_1^{(0)}<0$, and by \eqref{eq:v_1_linear_recursion} we get that it is monotonically increasing. 
        
        Using \lemref{lem:exp_recursion}, we have that $v_1^{(t)}-v_1^{(0)}\ge C\ln(t)$ for all $t\ge1$ and some $C>0$ that depends only on $\btheta^{(0)}$ and $\eta$. In particular, for some $T$ large enough such that $C\ln(T)\ge v_1^{(0)}$, it must hold that $v_1^{(T)}\ge0$, implying $|w_1^{(T)}|\ge b_1^{(T)}$. Namely, it holds that $-b_1^{(T)}/w_1^{(T)}\in[-1,1]$, and that the first neuron has turned inactive on one of the instances, concluding the proof of the proposition.
    \end{proof}


\section{Proofs for \secref{sec:slow_convergence}}

\subsection{Proof of \thmref{thm:no_PAC}}\label{app:no_PAC_proof}

    Before we prove the theorem, we will state and prove the following proposition, which proves that almost all initializations $\btheta^{(0)}$ such that $\Lcal(\btheta^{(0)})<0.5$, the limiting equilibrium point GD converges to is nonzero. Additionally, this proposition also allows us to ignore certain pathological edge cases of Lebesgue measure zero.
    \begin{proposition}\label{prop:measure_zero}
        Consider optimizing the objective in \eqref{eq:obj} using GD with step size $\eta\in(0,0.5]$. Then, for almost all $\btheta^{(0)}$ such that $\Lcal(\btheta^{(0)})<0.5$ and for all $t\ge0$, the GD iterates $\btheta^{(t)}=(w_1^{(t)},b_1^{(t)},w_2^{(t)},b_2^{(t)})$ satisfy
        \begin{enumerate}
            \item\label{item:n1}
            $|w_1^{(t)}|\neq|b_1^{(t)}|$,
            \item\label{item:n2}
            $|w_2^{(t)}|\neq|b_2^{(t)}|$,
            \item\label{item:nonzero_weights}
            $w_1^{(t)}\neq0$ and $w_2^{(t)}\neq0$.
            \item\label{item:equilibrium}
            $w_1^{(t)} - b_1^{(t)} + w_2^{(t)} + b_2^{(t)}\neq 0$,
            \item\label{item:loss}
            $\Lcal\p{\btheta^{(t)}}<0.5$.
        \end{enumerate}    
    \end{proposition}

    \begin{proof}
        We will first prove that if all items hold at an iteration $t$, then the last item also holds at the next iteration $t+1$.
    
        The assumption $\Lcal\p{\btheta^{(t)}}<0.5$ and \lemref{lem:loss_at_least_half} force the first neuron to be active on $x_2$ and the second neuron to be active on $x_1$. Thus, the only possible activation patterns for the network at such an iteration are when the first neuron is active only on $x_2$ or on both instances, and that the second neuron is active on $x_1$ or on both instances, for a total of four possible patterns (ignoring degenerate cases where breakpoints are ill-defined or coincide with data points due to Items~(\ref{item:n1}--\ref{item:nonzero_weights})). 
        \begin{itemize}
            \item
            Starting from the case where each neuron is active on one instance, we compute $\Phi(\btheta^{(t)}; x)$ at $x_1=-1$ and $x_2=1$ to obtain
            \begin{align*}
                \Phi(\btheta^{(t)};x_1) &= \relu{-w_1^{(t)} + b_1^{(t)}} - \relu{-w_2^{(t)} + b_2^{(t)}} = w_2^{(t)} - b_2^{(t)}, \\
                \Phi(\btheta^{(t)}, x_2) &= \relu{w_1^{(t)}+b_1^{(t)}} - \relu{w_2^{(t)} + b_2^{(t)}} = w_1^{(t)}+b_1^{(t)}.
            \end{align*}
            The loss at iteration $t$ is thus given by
            \begin{equation}\label{eq:loss_one_one}
                \Lcal\p{\btheta^{(t)}} = \frac{1}{2}\exp\p{w_2^{(t)}-b_2^{(t)}} + \frac{1}{2}\exp\p{-w_1^{(t)}-b_1^{(t)}}.
            \end{equation}
            With the above, we compute the partial derivatives as follows
            \begin{align*}
                \frac{\partial\Lcal(\btheta^{(t)})}{\partial w_1^{(t)}} &= -\frac{1}{2}\exp\p{-w_1^{(t)}-b_1^{(t)}}, \\
                \frac{\partial \Lcal(\btheta^{(t)})}{\partial b_1^{(t)}} &= -\frac{1}{2}\exp\p{-w_1^{(t)}-b_1^{(t)}}, \\
                \frac{\partial \Lcal(\btheta^{(t)})}{\partial w_2^{(t)}} &= \frac{1}{2}\exp\p{w_2^{(t)} - b_2^{(t)}}, \\
                \frac{\partial \Lcal(\btheta{(t)})}{\partial b_2^{(t)}} &= -\frac{1}{2}\exp\p{w_2^{(t)} - b_2^{(t)}},
            \end{align*}
            which imply the following gradient descent updates
            \begin{align*}
                w_1^{(t+1)} &= w_1^{(t)} + \frac{1}{2}\eta \exp\p{-w_1^{(t)} - b_1^{(t)}}, \\
                b_1^{(t+1)} &= b_1^{(t)} + \frac{1}{2}\eta \exp\p{-w_1^{(t)} - b_1^{(t)}}, \\
                w_2^{(t+1)} &= w_2^{(t)} - \frac{1}{2}\eta \exp\p{w_2^{(t)} - b_2^{(t)}}, \\
                b_2^{(t+1)} &= b_2^{(t)} + \frac{1}{2}\eta \exp\p{w_2^{(t)} - b_2^{(t)}}.
            \end{align*}
            Plugging the above in \eqref{eq:loss_one_one}, we get the update rule for the loss
            \[
                \Lcal\p{\btheta^{(t+1)}} = \frac{1}{2}\exp\p{w_2^{(t)}-b_2^{(t)} - \eta \exp\p{w_2^{(t)} - b_2^{(t)}}} + \frac{1}{2}\exp\p{-w_1^{(t)}-b_1^{(t)} -\eta \exp\p{-w_1^{(t)} - b_1^{(t)}}},
            \]
            namely, the loss over each instance strictly decreases, and therefore $\Lcal(\btheta^{(t+1)})<\Lcal(\btheta^{(t)})<0.5$.
    
            \item 
            Now suppose that the first neuron is active on both instances, and the second neuron is only active on $x_1$. Then the outputs of the network on the training points are given by
            \begin{align*}
                &\Phi(\btheta, x_1) = \relu{-w_1 + b_1} - \relu{-w_2 + b_2} = -w_1 + b_1 +w_2 - b_2, \\
                &\Phi(\btheta; x_2) = \relu{w_1 + b_1} - \relu{w_2 + b_2}=  w_1 + b_1,
            \end{align*}
            and the training loss is therefore
            \begin{equation}\label{eq:loss_two_one}
                \Lcal(\btheta) = \frac12\exp\p{-w_1+b_1+w_2-b_2} + \frac12\exp\p{-w_1 - b_1}.
            \end{equation}
            With the above, we can compute the partial derivatives to obtain
            \begin{align*}
                \frac{\partial \Lcal(\btheta)}{\partial w_1} &= -\frac12\exp(-w_1 + b_1 + w_2 - b_2) - \frac12\exp(-w_1-b_1), \\
                \frac{\partial \Lcal(\btheta)}{\partial b_1} &= \frac12\exp(-w_1 + b_1 + w_2 - b_2) - \frac12\exp(-w_1 - b_1), \\
                \frac{\partial \Lcal(\btheta)}{\partial w_2} &= \frac12\exp(-w_1 + b_1 + w_2 - b_2), \\
                \frac{\partial \Lcal(\btheta)}{\partial b_2} &= -\frac12\exp(-w_1 + b_1 + w_2 - b_2).
            \end{align*}
            Thus, the GD updates are given by
            \begin{align*}
                w_1^{(t+1)} &= w_1^{(t)} + \frac12\eta \exp\p{-w_1^{(t)} + b_1^{(t)} + w_2^{(t)} - b_2^{(t)}} + \frac12\eta\exp\p{-w_1^{(t)}-b_1^{(t)}}, \\
                b_1^{(t+1)} &= b_1^{(t)} - \frac12\eta \exp\p{-w_1^{(t)} + b_1^{(t)} + w_2^{(t)} - b_2^{(t)}} + \frac12\eta\exp\p{-w_1^{(t)} - b_1^{(t)}}, \\
                w_2^{(t+1)} &= w_2^{(t)} - \frac12\eta \exp\p{-w_1^{(t)} + b_1^{(t)} + w_2^{(t)} - b_2^{(t)}}, \\
                b_2^{(t+1)} &= b_2^{(t)} + \frac12\eta \exp\p{-w_1^{(t)} + b_1^{(t)} + w_2^{(t)} - b_2^{(t)}}.
            \end{align*}
            Plugging the above in \eqref{eq:loss_two_one}, we get the update rule for the loss over the first instance
            \[
                \ell\p{-\Phi\p{\btheta^{(t+1)};x_1}} = \exp\p{-w_1^{(t)} + b_1^{(t)} + w_2^{(t)} - b_2^{(t)} - 2\eta \exp\p{-w_1^{(t)} + b_1^{(t)} + w_2^{(t)} - b_2^{(t)}}},
            \]
            and for the second instance
            \[
                \ell\p{\Phi\p{\btheta^{(t+1)};x_2}} = \exp\p{-w_1^{(t)} - b_1^{(t)} -\eta\exp\p{-w_1^{(t)} - b_1^{(t)}}}.
            \]
            Therefore the training loss is strictly decreasing, and we have $\Lcal(\btheta^{(t+1)})<\Lcal(\btheta^{(t)})<0.5$.
    
            \item 
            Suppose that the first neuron is active on $x_2$, and the second neuron is active on both instances. Then the outputs of the network on the training points are given by
            \begin{align*}
                \Phi(\btheta, x_1) &= \relu{-w_1 + b_1} - \relu{-w_2 + b_2} = w_2 - b_2, \\
                \Phi(\btheta; x_2) &= \relu{w_1 + b_1} - \relu{w_2 + b_2}=  w_1 + b_1 - w_2 - b_2,
            \end{align*}
            and the training loss is therefore
            \begin{equation}\label{eq:loss_one_two}
                \Lcal(\btheta) = \frac12\exp(w_2 - b_2) + \frac12\exp(w_2+b_2-w_1-b_1).
            \end{equation}
            With the above, we can compute the partial derivatives to obtain
            \begin{align*}
                \frac{\partial \Lcal(\btheta)}{\partial w_1} &= -\frac12\exp(w_2+b_2-w_1-b_1), \\
                \frac{\partial \Lcal(\btheta)}{\partial b_1} &= -\frac12\exp(w_2+b_2-w_1-b_1), \\
                \frac{\partial \Lcal(\btheta)}{\partial w_2} &= \frac12\exp(w_2-b_2) + \frac12\exp(w_2+b_2-w_1-b_1), \\
                \frac{\partial \Lcal(\btheta)}{\partial b_2} &= -\frac12\exp(w_2 - b_2) + \frac12\exp(w_2+b_2-w_1-b_1).
            \end{align*}
            Thus, the GD updates are given by
            \begin{align*}
                w_1^{(t+1)} &= w_1^{(t)} + \frac12\eta\exp\p{w_2^{(t)}+b_2^{(t)}-w_1^{(t)}-b_1^{(t)}}, \\
                b_1^{(t+1)} &= b_1^{(t)} + \frac12\eta \exp\p{w_2^{(t)}+b_2^{(t)}-w_1^{(t)}-b_1^{(t)}}, \\
                w_2^{(t+1)} &= w_2^{(t)} - \frac12\eta\exp\p{w_2^{(t)}-b_2^{(t)}} - \frac12\eta\exp\p{w_2^{(t)}+b_2^{(t)}-w_1^{(t)}-b_1^{(t)}}, \\
                b_2^{(t+1)} &= b_2^{(t)} +\frac12\eta\exp\p{w_2^{(t)} - b_2^{(t)}} - \frac12\eta\exp\p{w_2^{(t)}+b_2^{(t)}-w_1^{(t)}-b_1^{(t)}}. 
            \end{align*}
            Plugging the above in \eqref{eq:loss_one_two}, we get the update rule for the loss over the first instance
            \[
                \ell\p{-\Phi\p{\btheta^{(t+1)};x_1}} = \exp\p{w_2^{(t)}-b_2^{(t)} - \eta \exp\p{w_2^{(t)} - b_2^{(t)}}},
            \]
            and for the second instance
            \[
                \ell\p{\Phi\p{\btheta^{(t+1)};x_2}} = \frac12\exp\p{w_2^{(t)}+b_2^{(t)}-w_1^{(t)}-b_1^{(t)} -2\eta\exp\p{w_2^{(t)}+b_2^{(t)}-w_1^{(t)}-b_1^{(t)}}}.
            \]
            Therefore the training loss is strictly decreasing, and we have $\Lcal(\btheta^{(t+1)})<\Lcal(\btheta^{(t)})<0.5$.

            \item
            Suppose that both neurons are active on both instances. Then the outputs of the network on the training points are given by
            \begin{align*}
                &\Phi\p{\btheta^{(t)};x_1} = \relu{-w_1^{(t)} + b_1^{(t)}} - \relu{-w_2^{(t)} + b_2^{(t)}} = -w_1^{(t)}+b_1^{(t)} + w_2^{(t)} - b_2^{(t)}, \\
                &\Phi\p{\btheta^{(t)};x_2} = \relu{w_1^{(t)} + b_1^{(t)}} - \relu{w_2^{(t)} + b_2^{(t)}} = w_1^{(t)}+b_1^{(t)} - w_2^{(t)} - b_2^{(t)}.
            \end{align*}
            The loss at iteration $t$ is thus given by
            \begin{equation}\label{eq:loss_linear_setting}
                \Lcal\p{\btheta^{(t)}} = \frac{1}{2}\exp\p{-w_1^{(t)}+b_1^{(t)} + w_2^{(t)} - b_2^{(t)}} + \frac{1}{2}\exp\p{-w_1^{(t)}-b_1^{(t)} + w_2^{(t)} + b_2^{(t)}}.
            \end{equation}
            Using the shorthands $A_t\coloneqq\exp\p{-w_1^{(t)}+b_1^{(t)} + w_2^{(t)} - b_2^{(t)}}$ and $B_t\coloneqq\exp\p{-w_1^{(t)}-b_1^{(t)} + w_2^{(t)} + b_2^{(t)}}$, the partial derivatives can be written compactly as
            \begin{align*}
                \frac{\partial\Lcal(\btheta^{(t)})}{\partial w_1^{(t)}} &= -\frac{1}{2}A_t - \frac{1}{2}B_t, \\
                \frac{\partial \Lcal(\btheta^{(t)})}{\partial b_1^{(t)}} &= \frac{1}{2}A_t - \frac{1}{2}B_t, \\
                \frac{\partial \Lcal(\btheta^{(t)})}{\partial w_2^{(t)}} &= \frac{1}{2}A_t + \frac{1}{2}B_t, \\
                \frac{\partial \Lcal(\btheta{(t)})}{\partial b_2^{(t)}} &= - \frac{1}{2}A_t + \frac{1}{2}B_t.
            \end{align*}
            Using the above to compute the gradient step update for each parameter, we get
            \begin{align*}
                w_1^{(t+1)} &= w_1^{(t)} + \frac{1}{2}\eta A_t + \frac{1}{2}\eta B_t, \\
                b_1^{(t+1)} &= b_1^{(t)} - \frac{1}{2}\eta A_t + \frac{1}{2}\eta B_t, \\
                w_2^{(t+1)} &= w_2^{(t)} - \frac{1}{2}\eta A_t - \frac{1}{2}\eta B_t,\nonumber\\
                b_2^{(t+1)} &= b_2^{(t)} + \frac{1}{2}\eta A_t - \frac{1}{2}\eta B_t.\nonumber
            \end{align*}
            Substituting the above in \eqref{eq:loss_linear_setting}, we can write the loss update rule for the first instance at iteration $t+1$ as
            \begin{align*}
                \ell\p{\Phi\p{\btheta^{(t+1)};x_1},y_1} &= \exp\p{-w_1^{(t+1)}+b_1^{(t+1)} + w_2^{(t+1)} - b_2^{(t+1)}}\\
                &= \exp\p{-w_1^{(t)}+b_1^{(t)} + w_2^{(t)} - b_2^{(t)} - 2\eta A_t} \\
                &= \exp\p{-2\eta A_t}\cdot\ell\p{\Phi\p{\btheta^{(t)};x_1},y_1},
            \end{align*}
            and since $A_t>0$, we obtain that the loss on $x_1$ strictly decreases with $t$.
            Likewise, we can verify that
            \[
                \ell\p{\Phi\p{\btheta^{(t+1)};x_2},y_2} = \exp\p{-2\eta B_t}\cdot\ell\p{\Phi\p{\btheta^{(t)};x_2},y_2},
            \]
            and therefore $\Lcal(\btheta^{(t+1)}) < \Lcal(\btheta^{(t)}) <0.5$. 
        \end{itemize}

        Fix any $\eta\in(0,0.5]$, and define the GD map $G:\reals^4\to\reals^4$ by $G(\btheta)\coloneqq\btheta-\eta\nabla\Lcal(\btheta)$. We will show that if $\Lcal(\btheta)<0.5$ then this is an analytic open map by analyzing its Jacobian, $J_G$. We have
        \[
            J_G(\btheta) = I_4-\eta\nabla^2\Lcal(\btheta),
        \]
        where $I_4$ is the four-dimensional identity matrix. For ease of notation, let $\theta_1\coloneqq w_1$, $\theta_2\coloneqq b_1$, $\theta_3\coloneqq w_2$ and $\theta_4\coloneqq b_2$. We now bound the entries of the Hessian as follows: For any activation pattern and any $i\in\{1,2\}$, $\Phi(\btheta;x_i)$ computes an affine transformation of the form $\sum_{j=1}^4\xi_{i,j}\theta_j$ for some $\xi_{i,j}\in\{-1,0,1\}$, $j\in[4]$. From this, we obtain that $G$ is analytic (it is the difference between the identity map and a composition between an affine transformation and an exponent) wherever it is differentiable, and we also have that for all $j_1,j_2\in[4]$,
        \[
            \frac{\partial^2}{\partial\theta_{j_1}\partial\theta_{j_2}}\Lcal(\btheta) = \frac{1}{2}\xi_{1,j_1}\xi_{1,j_2}\exp\p{\Phi\p{\btheta;x_1}} + \frac{1}{2}\xi_{2,j_1}\xi_{2,j_2}\exp\p{-\Phi\p{\btheta;x_2}}.
        \]
        Bounding the above in absolute value yields
        \[
            \abs{\frac{\partial^2}{\partial\theta_{j_1}\partial\theta_{j_2}}\Lcal(\btheta)} \le \frac{1}{2}\exp\p{\Phi\p{\btheta;x_1}} + \frac{1}{2}\exp\p{-\Phi\p{\btheta;x_2}} = \Lcal(\btheta) < 0.5.
        \]
        Thus, for $\eta\in(0,0.5]$ we have that $J_G(\btheta)$ is strictly diagonally dominant, and hence non-singular. This implies that $G$ is an open map on any domain where it is differentiable.
    
        We now prove Items~(\ref{item:n1}--\ref{item:loss}) by induction. First, we have that $\Lcal(\btheta^{(0)})<0.5$ holds by our assumption, and that if \itemref{item:loss} holds for $t$ then it also holds for $t+1$ by the previous case analysis.
        
        By removing the measure zero set where Items~(\ref{item:n1}--\ref{item:equilibrium}) are violated, we have that the induction base holds for $t=0$. For the induction step, assume that these items hold for $t$. Since $\Lcal(\btheta^{(t)})<0.5$, we have that $G$ is an analytic open map, and its composition with itself $t+1$ times, $G^{t+1} \coloneqq G\circ\ldots\circ G$, is also an analytic open map. Thus, the maps 
        \begin{itemize}
            \item
            $\beta_{1,\pm}^{t+1}(\btheta^{(0)})\coloneqq G_1^{t+1}(\btheta^{(0)})\pm G_2^{t+1}(\btheta^{(0)})$,
            \item 
            $\beta_{2,\pm}^{t+1}(\btheta^{(0)})\coloneqq G_3^{t+1}(\btheta^{(0)})\pm G_4^{t+1}(\btheta^{(0)})$,
            \item 
            $S^{t+1}\coloneqq G_1^{t+1}(\btheta^{(0)})-G_2^{t+1}(\btheta^{(0)}) + G_3^{t+1}(\btheta^{(0)}) + G_4^{t+1}(\btheta^{(0)})$,
            \item 
            $G_1^{t+1}(\btheta^{(0)})$ and $G_3^{t+1}(\btheta^{(0)})$,
        \end{itemize}
        where $G^{k+1}_j$ returns the $j^{\textrm{th}}$ coordinate of $G^{t+1}$, are all analytic. Moreover, none of these can be the zero map, since $G^{t+1}$ is an open map, and thus any open set is mapped by $G^{t+1}$ to an open set $U$. Let $\btheta=(w_1,b_1,w_2,b_2)\in U$ be arbitrary, we can choose $\varepsilon>0$ small enough such that $(w_1+\varepsilon,b_1,w_2,b_2)\in U$, and thus either one of $w_1-b_1+w_2+b_2\neq0$ or $w_1+\varepsilon-b_1+w_2+b_2\neq0$ hold, and thus $S^{t+1}$ cannot collapse an open set to zero, and it is not the zero map. A similar argument proves that the remaining maps are not the zero map either, and therefore it holds that the zero-set (the preimage of zero) of these maps has measure zero \citep{mityagin2015zero}. By removing these measure-zero sets from the set of valid initializations for $\btheta^{(0)}$, we have that Items~(\ref{item:n1}--\ref{item:equilibrium}) hold. Finally, by taking an infinite yet countable union over the measure-zero sets removed at each iteration of GD, we are removing a set of measure zero from the set $\{\btheta^{(0)}:\Lcal(\btheta^{(0)})<0.5\}$, which concludes the proof of the proposition.
    \end{proof}

    With the above proposition, we turn to prove our main theorem.

    \begin{proof}[Proof of \thmref{thm:no_PAC}]
        By \propref{prop:measure_zero}, we have for almost all $\btheta^{(0)}$ such that $\Lcal(\btheta^{(0)})<0.5$, that the GD iterates $\btheta^{(t)}$ satisfy Items~(\ref{item:n1}--\ref{item:loss}). In particular, the breakpoints $\beta_1^{(t)},\beta_2^{(t)}$ are well-defined and never overlap a data instance for all $t\ge0$, the loss at any iteration $t\ge0$ satisfies $\Lcal(\btheta^{(t)})<0.5$, and $v_1^{(t)}+u_2^{(t)}\neq 0$.

        By the previous conditions and \lemref{lem:loss_at_least_half}, the first neuron must be active on $x_2$ and the second neuron must be active on $x_1$. In addition, each neuron may or may not be active on the remaining instance. We split our analysis according to the activation pattern as follows:
        \begin{enumerate}
            \item\label{item:both_active_on_both}
            If both neurons are active on both instances, then \asmref{asm:linear_setting} holds. By \propref{prop:linear_setting}, after a constant number of iterations, we transition to the setting of \asmref{asm:one_active_both_one_active_one_case_3} or \asmref{asm:one_active_both_one_active_one_case_4}, which is analyzed in \itemref{item:wrong_sign}; to the setting of \asmref{asm:one_active_both_one_active_one_case_1} or \asmref{asm:one_active_both_one_active_one_case_2}, which is analyzed in \itemref{item:right_sign}; to the setting of \asmref{asm:neurons_active_only_on_one_point}, which is analyzed in \itemref{item:specialized_dead_interval}; or to the setting of \asmref{asm:base_case}, which is analyzed in \itemref{item:specialized_with_intersection}.
            
            \item\label{item:wrong_sign}
            If both neurons point to the same direction ($\sign(w_1^{(t)})=\sign(w_2^{(t)})$), then \asmref{asm:one_active_both_one_active_one_case_3} or \asmref{asm:one_active_both_one_active_one_case_4} holds. By Propositions~\ref{prop:one_active_both_one_active_one_case_3},\ref{prop:one_active_both_one_active_one_case_4}, after a constant number of iterations, we transition to the setting of \asmref{asm:one_active_both_one_active_one_case_1} or \asmref{asm:one_active_both_one_active_one_case_2}, which is analyzed in \itemref{item:right_sign}; to the setting of \asmref{asm:neurons_active_only_on_one_point}, which is analyzed in \itemref{item:specialized_dead_interval}; or to the setting of \asmref{asm:base_case}, which is analyzed in \itemref{item:specialized_with_intersection}.
            
            \item\label{item:right_sign}
            If one neuron has specialized and the other points to the opposite direction and is active on both instances, then \asmref{asm:one_active_both_one_active_one_case_1} or \asmref{asm:one_active_both_one_active_one_case_2} holds. By Propositions~\ref{prop:one_active_both_one_active_one_case_1},\ref{prop:one_active_both_one_active_one_case_2}, after a constant number of iterations, we transition to the setting of \asmref{asm:neurons_active_only_on_one_point}, which is analyzed in \itemref{item:specialized_dead_interval}; or to the setting of \asmref{asm:base_case}, which is analyzed in \itemref{item:specialized_with_intersection}.
            
            \item\label{item:specialized_dead_interval}
            If both neurons have specialized and are inactive on a subinterval of $(-1,1)$, \asmref{asm:neurons_active_only_on_one_point} holds. By \propref{prop:specialized_dead_interval}, after a constant number of iterations, we transition to the setting of \asmref{asm:base_case}, which is analyzed in \itemref{item:specialized_with_intersection}.

            \item\label{item:specialized_with_intersection}
            Finally, if both neurons have specialized and are active on a subinterval of $(-1,1)$, \asmref{asm:base_case} holds. By \propref{prop:slow_margin_convergence}, we have
            \[
                x^\star(t) = \Theta\p{\frac{1}{\ln(t)}}. 
            \]
        \end{enumerate}
        Since the settings in Items~(\ref{item:both_active_on_both}--\ref{item:specialized_dead_interval}) can only hold for a constant number of iterations before transitioning to one of the next items, after a constant number of iterations, we transition to the final setting in \itemref{item:specialized_with_intersection}, from which the theorem follows.
    \end{proof}

    \subsection{Proof of \thmref{thm:non_asymptotic}}\label{app:proof_non_asymptotic}

    Before we prove the theorem, we will need a few auxiliary lemmas. The following lemma provides explicit constants for bounding the growth rate of the weights.
    
    \begin{lemma}\label{lem:weight_dif_ubound}
        Under Assumption~\ref{asm:base_case} where all parameters are initialized using the He initialization (and therefore $b_1^{(0)} = b_2^{(0)} = 0$) and under the assumption that $w_1^{(0)} \in (0, 3)$ and $w_2^{(0)} \in (-2, 0)$ we have that
        \[ w_1^{(t)} - w_2^{(t)} \leq 5 + \ln \left(1 + 2 \exp\p{\eta} \eta  t \right). \]
        
    \end{lemma}
    \begin{proof}
        Using \lemref{lem:converge_rate_w1} and \lemref{lem:converge_rate_w2} we get
        \[ w_1^{(t)} - w_2^{(t)} \leq \frac{1}{2} \ln \left( \exp\p{2w_1^{(0)}} + 2 \eta C_1 \exp\p{v_1^{(0)}} \cdot t \right) +  \frac{1}{2} \ln \left( \exp\p{-2w_2^{(0)}} +  2\eta \exp\p{-u_2^{(0)}} C_2 \cdot t\right)\]
        where $C_1 = \exp \left( \eta \exp\p{v_1^{(0)}} \exp\p{-2w_1^{(0)}} \right)$ and $C_2 = \exp \left( \eta \exp\p{-u_2^{(0)}} \exp\p{2w_2^{(0)}} \right)$.
        Since $b_1^{(0)} = b_2^{(0)} = 0$ we can simplify the constant to get
        \[ w_1^{(t)} - w_2^{(t)} \leq \frac{1}{2} \ln \left( \exp\p{2w_1^{(0)}} + 2 \eta C_1 \exp\p{w_1^{(0)}} \cdot t \right) +  \frac{1}{2} \ln \left( \exp\p{-2w_2^{(0)}} +  2\eta \exp\p{-w_2^{(0)}} C_2 \cdot t\right)\]
        where $C_1 = \exp \left( \eta \exp\p{-w_1^{(0)}} \right)$ and $C_2 = \exp \left( \eta \exp\p{w_2^{(0)}} \right)$.
        Merging the $\ln$ functions we get
        \[ w_1^{(t)} - w_2^{(t)} \leq \frac{1}{2} \ln \left( \left[\exp\p{2w_1^{(0)}} + 2 \eta C_1 \exp\p{w_1^{(0)}} \cdot t \right] \left[ \exp\p{-2w_2^{(0)}} +  2\eta \exp\p{-w_2^{(0)}} C_2 \cdot t \right] \right) \]
        Factor the leading exponential in each term:
        \begin{align*}
            & \exp\p{2w_1^{(0)}} + 2 \eta C_1 \exp\p{w_1^{(0)}} \cdot t = \exp\p{2w_1^{(0)}} \left(1 +  2 \eta C_1 \exp\p{-w_1^{(0)}} \cdot t \right) \\
            & \exp\p{-2w_2^{(0)}} +  2\eta \exp\p{-w_2^{(0)}} C_2 \cdot t = \exp\p{-2w_2^{(0)}} \left( 1 + 2\eta \exp\p{w_2^{(0)}} C_2 \cdot t \right)
        \end{align*}
        Plugging it in the bound we get
        \[ w_1^{(t)} - w_2^{(t)} \leq \frac{1}{2} \ln \left( \left[ \exp\p{2w_1^{(0)} - 2w_2^{(0)}} \right] \left[1 +  2 \eta C_1 \exp\p{-w_1^{(0)}} \cdot t \right] \left[ 1 + 2\eta \exp\p{w_2^{(0)}} C_2 \cdot t \right] \right). \]
        Once again by using logarithm rules we get
        \[ w_1^{(t)} - w_2^{(t)} \leq w_1^{(0)} - w_2^{(0)} + \frac{1}{2} \ln \left( \left[1 +  2 \eta C_1 \exp\p{-w_1^{(0)}} \cdot t \right] \left[ 1 + 2\eta \exp\p{w_2^{(0)}} C_2 \cdot t \right] \right). \]
        Using the inequality $(1+A)(1+B) \leq (1 + \max\{A, B \})^2$ for all $A,B >0$ we get
        \[ w_1^{(t)} - w_2^{(t)} \leq w_1^{(0)} - w_2^{(0)} + \ln \left(1 + \eta A t \right). \]
        where $A = \max\{2 C_1 \exp\p{-w_1^{(0)}}, 2 C_2 \exp\p{w_2^{(0)}}\}$.
        Using the assumptions that $w_1^{(0)} \in (0,3)$ and $w_2^{(0)} \in (-2, 0)$ we get $w_1^{(0)} - w_2^{(0)} < 5$ and $A < 2 \exp\p{\eta}$ and thus
        \[ w_1^{(t)} - w_2^{(t)} \leq 5 + \ln \left(1 + 2 \exp\p{\eta} \eta  t \right). \]
    \end{proof}

    The following lemma shows that under a zero biases assumption and a tamed step size, the convergence rate of the system to a balanced GD step for both neurons is monotone.

    \begin{lemma}\label{lem:monotone_limit}
        Under \asmref{asm:base_case}, suppose that we optimize the objective in \eqref{eq:obj} using GD with step size $\eta\in(0,0.5]$. Further assume that $b_1^{(0)}=b_2^{(0)}$. Then,
        \[
            \lim _{t \to \infty} b_2^{(t)} - b_1^{(t)} = \lim _{t \to \infty} w_1^{(t)} + w_2^{(t)} = \frac{v_1^{(0)} + u_2^{(0)}}{2}.
        \]
        where the convergence is monotone.
    \end{lemma}
    \begin{proof}
        First, by \lemref{lem:limit_of_sum} we have that the limits hold, and therefore it is only left to prove monotonicity. Since $w_1^{(t)} + w_2^{(t)} - b_1^{(t)} + b_2^{(t)} = v_1^{(0)} + u_2^{(0)}$ for all $t\ge0$ by Equations~(\ref{eq:w_1_simple_dynamics}--\ref{eq:b_2_simple_dynamics}), it suffices to prove monotonicity for just the sum of the weights.
        
        For the sake of brevity, we define $S_t \coloneqq w_1^{(t)} + w_2^{(t)}$. Compute
        \begin{align*}
            S_{t+1} &= w_1^{(t+1)} + w_2^{(t+1)} \\
            &= w_1^{(t)} + \eta \exp(-2w_1^{(t)} + v_1^{(0)}) + w_2^{(t)} - \eta \exp(2w_2^{(t)} - u_2^{(0)}) \\
            &= S_t + \eta \left[ \exp(-2w_1^{(t)} + v_1^{(0)}) - \exp(2w_2^{(t)} - u_2^{(0)}) \right]
        \end{align*}
        Define $D_{t} \coloneqq w_1^{(t)} - w_2^{(t)} > 0$. We can rewrite $w_1^{(t)} = \frac{1}{2}(S_t + D_t), ~ w_2^{(t)} = \frac{1}{2}(S_t - D_t)$, and thus we can also rewrite $S_{t+1} - S_t$ using $v_1^{(0)}, u_2^{(0)}, S_t$ and $D_t$ as follows
        \begin{align*}
            S_{t+1} - S_t &= \eta \left[ \exp(-2w_1^{(t)} + v_1^{(0)}) - \exp(2w_2^{(t)} - u_2^{(0)}) \right] \\
            &= \eta \left[ \exp(-S_t -D_t + v_1^{(0)}) - \exp(S_t - D_t - u_2^{(0)}) \right] \\
            &= \eta \exp(-D_t)\left[ \exp(v_1^{(0)} - S_t) - \exp(S_t - u_2^{(0)})\right] \\
            &= \eta \exp(-D_t) g(S_t)
        \end{align*}
        where $g(s) \coloneqq \exp(v_1^{(0)}-s) - \exp(s - u_2^{(0)})$. We notice that if $S_t = S^\star = \frac{v_1^{(0)} + u_2^{(0)}}{2}$, then $S_{t+1} = S_t$, implying that this is a fixed point. Since $\eta \exp(-D_t) > 0$, the only term that affects the sign of $S_{t+1} - S_t$ is $g(S_t)$. If $S_t < S^\star$ then $g(S_t) < 0$ which implies that $S_{t+1} > S_t$; and if $S_t > S^\star$, then $g(S_t) > 0$, which implies that $S_{t+1} < S_t$. Assume without loss of generality that $S_t < S^\star$ (the other direction is symmetric). We have that $g(s)$ is decreasing since 
        \[
            g^\prime(s) = -\exp(v_1^{(0)} - s) - \exp(s - u_2^{(0)}) < 0.
        \]
        Denote $L \coloneqq \sup_{s \in [S_0, S^\star]}|g^\prime(s)|$. We will now compute $L$. We have $g^{\prime\prime}(s) = \exp(v_1^{(0)} - s) - \exp(s - u_2^{(0)})$. Equating this to zero to find the extrema points, we get that $S^\star$ is an extremum point of $g^\prime(s)$. Since 
        \[
            g^{(3)}(s) = -\exp(v_1^{(0)} - s) - \exp(s - u_2^{(0)}) < 0,    
        \]
        then $S^\star$ is a maximum point of $g^\prime(s)$. But since $g^\prime(s) < 0$, then $S^\star$ is a minimum point of $|g^\prime(s)|$. This implies that the maximum point in the interval must be attained at the boundary, thus
        \[
            L = \max\{\exp(v_1^{(0)}-S_0) + \exp(S_0 - u_2^{(0)}), \exp(v_1^{(0)} - S^\star) + \exp(S^\star - u_2^{(0)}) \}.    
        \]
        Since $g(S)$ is $L$-Lipschitz, we get
        \begin{align*}
            |g(S_t) - g(S^\star)| \leq L |S_t - S^\star|.
        \end{align*}
        But $g(S)$ is also decreasing, and $g(S^\star) = 0$ and $g(S_t) \geq g(S^\star) = 0$, so we get that
        \begin{align*}
            g(S_t) &= g(S_t) - g(S^\star) \leq L(S^\star - S_t),
        \end{align*}
        and thus by the above equation we compute
        \[
            S_{t+1}-S_t = \eta \exp(-D_t) g(S_t) \leq \eta L (S^\star - S_t) \exp(-D_t).
        \]
        Rearranging yields
        \[
            S_{t+1} < \eta L \exp(-D_t) S^\star + (1-\eta L \exp(-D_t)) S_t,
        \]
        and further simplifying, we have
        \[
            S_{t+1} - S^\star < (1-\exp(-D_t) \eta L)(S_t - S^\star).
        \]
        
        Combining Assumption~\ref{asm:base_case} and Proposition~\ref{prop:base_case_invariance}, we have that $w_1^{(t)} > 0$ and increases, and that $w_2^{(t)} < 0$ and decreases, hence $\exp(-D_t)$ decreases, and we can upper bound it by $\exp(-D_0) = \exp(w_2^{(0)} - w_1^{(0)})$.
        Plugging this upper bound in the above displayed equation, we get
        \begin{equation}\label{eq:S_t_monotone}
            S_{t+1} - S^\star < (1 - \eta L \exp(w_2^{(0)} - w_1^{(0)})) (S_t - S^\star).
        \end{equation}
        Next, we upper bound $L$ as follows. The assumption that $b_1^{(0)}=b_2^{(0)}=0$ yields
        \[
            \exp\p{v_1^{(0)} - S^\star} + \exp(S^\star - u_2^{(0)}) = 2\exp\p{\frac{w_2^{(0)}-w_1^{(0)}}{2}} \le \exp\p{w_2^{(0)}-w_1^{(0)}}.
        \]
        Similarly, we have
        \[
            \exp\p{v_1^{(0)}-S_0} + \exp\p{S_0 - u_2^{(0)}} = \exp\p{-w_2^{(0)}} + \exp\p{w_1^{(0)}}.
        \]
        Thus, when the above is the maximum in $L$, we have
        \[
            \frac{\exp(w_1^{(0)} - w_2^{(0)})}{L} = \p{\exp\p{w_2^{(0)}} + \exp\p{-w_1^{(0)}}}^{-1} \ge \frac12,
        \]
        and if the maximum in $L$ is attained by the second argument, we have
        \[
            \frac{\exp(w_1^{(0)} - w_2^{(0)})}{L} \ge 1.
        \]
        Therefore, in both cases, we have by our assumption that $\eta\le 0.5 \le \frac{\exp(w_1^{(0)} - w_2^{(0)})}{L}$, implying that $1-\eta L \exp(w_2^{(0)} - w_1^{(0)}) \ge 0$, and thus by \eqref{eq:S_t_monotone} and by using the assumption that $S_t < S^\star$, we conclude that
        \[
            S_{t+1} \in (S_{t},S^\star].
        \]
    \end{proof}

    The following proposition shows that under a zero biases assumption, the difference between the biases throughout optimization increases at a linear rate to a constant that depends on the sum of the weights.
    
    \begin{proposition}\label{prop:non_asymptotic}
        Under \asmref{asm:base_case}, suppose that we optimize the objective in \eqref{eq:obj} using GD with step size $\eta\in(0,0.3]$. Suppose that $w_2^{(0)}<0<w_1^{(0)}$, $w_1^{(0)}+w_2^{(0)}>1$ and $b_1^{(0)}=b_2^{(0)}=0$. Then,
        \[
            b_2^{(t)} - b_1^{(t)} \ge \frac{w_1^{(0)}+w_2^{(0)}}{11}-\frac{1}{10}
        \]
        for all $t\ge t_0$, where
        \[
            t_0\coloneqq\floor{\frac{w_1^{(0)} + w_2^{(0)}}{2\eta \exp\p{\frac{-w_1^{(0)}+w_2^{(0)}}{2}}}}.
        \]
    \end{proposition}
    Consider $\frac{1}{2}\eta \exp\p{-w_1^{(t)} - b_1^{(t)}}$. By the monotonicity of the weights in \propref{prop:base_case_invariance}, 
    \[
        \frac{1}{2}\eta \exp\p{-w_1^{(t)} - b_1^{(t)}}\le \frac{1}{2}\eta \exp\p{-w_1^{(0)} - b_1^{(0)}} = \frac{1}{2}\eta \exp\p{-w_1^{(0)}}    
    \]
    for all $t$, so
    \begin{equation}\label{eq:b_1_ubound}
        b_1^{(t)} \le b_1^{(0)} + t\cdot\frac{1}{2}\eta \exp\p{-w_1^{(0)}} = t\cdot\frac{1}{2}\eta \exp\p{-w_1^{(0)}}.
    \end{equation}
    On the other hand, as long as $w_2^{(t)} - b_2^{(t)} \ge \frac{-w_1^{(0)}+w_2^{(0)}}{2}$, then the growth rate of $b_2^{(t)}$ can be lower bounded as follows
    \begin{equation}\label{eq:b_2_lbound}
        b_2^{(t)} = b_2^{(0)} +\frac12\sum_{i=0}^{t-1}\eta \exp\p{w_2^{(i)} - b_2^{(i)}} \ge t\cdot\frac{1}{2}\eta \exp\p{\frac{-w_1^{(0)}+w_2^{(0)}}{2}},
    \end{equation}
    and we also have by subtracting \eqref{eq:b_2_simple_dynamics} from \eqref{eq:w_2_simple_dynamics} that
    \[
        w_2^{(t+1)} - b_2^{(t+1)} = w_2^{(t)} - b_2^{(t)} - \eta \exp\p{w_2^{(t)} - b_2^{(t)}} \le w_2^{(t)} - b_2^{(t)} - \eta \exp\p{\frac{-w_1^{(0)}+w_2^{(0)}}{2}},
    \]
    which recursively implies that
    \[
        w_2^{(t)} - b_2^{(t)} \le w_2^{(0)} - b_2^{(0)} - t\eta \exp\p{\frac{-w_1^{(0)}+w_2^{(0)}}{2}} = w_2^{(0)} - t\eta \exp\p{\frac{-w_1^{(0)}+w_2^{(0)}}{2}}.
    \]
    By the above, $w_2^{(t)} - b_2^{(t)}$ decreases at a linear rate if it is lower bounded, which entails that a necessary condition for $w_2^{(t)} - b_2^{(t)} \ge \frac{-w_1^{(0)}+w_2^{(0)}}{2}$ to hold is
    \[
        \frac{-w_1^{(0)}+w_2^{(0)}}{2} \le w_2^{(0)} - t\eta \exp\p{\frac{-w_1^{(0)}+w_2^{(0)}}{2}}.
    \]
    Solving for $t$, we have
    \begin{equation}\label{eq:t_upper_bound}
        t \le \frac{w_1^{(0)} + w_2^{(0)}}{2\eta \exp\p{\frac{-w_1^{(0)}+w_2^{(0)}}{2}}}.
    \end{equation}
    The assumptions $w_2^{(0)}<0<w_1^{(0)}$ and $w_1^{(0)} + w_2^{(0)}>1$ imply $-w_1^{(0)} < w_2^{(0)}$, yielding that at the first iteration we have
    \[
        w_2^{(0)} - b_2^{(0)} = w_2^{(0)} \ge \frac{-w_1^{(0)}+w_2^{(0)}}{2},
    \]
    thus \eqref{eq:t_upper_bound} holds for at least one iteration. Moreover, since our assumptions also entail
    \[
        \frac{w_1^{(0)} + w_2^{(0)}}{2\eta \exp\p{\frac{-w_1^{(0)}+w_2^{(0)}}{2}}} > \frac{1}{2\eta \exp\p{\frac{-1+0}{2}}} > \frac{0.3}{\eta} \ge 1,
    \]
    hence we can take the floor function of the right-hand side in \eqref{eq:t_upper_bound},
    \begin{equation}\label{eq:t_0_bounds}
        \frac{w_1^{(0)} + w_2^{(0)}}{2\eta \exp\p{\frac{-w_1^{(0)}+w_2^{(0)}}{2}}} - 1 \le t_0\coloneqq \floor{\frac{w_1^{(0)} + w_2^{(0)}}{2\eta \exp\p{\frac{-w_1^{(0)}+w_2^{(0)}}{2}}}} \le \frac{w_1^{(0)} + w_2^{(0)}}{2\eta \exp\p{\frac{-w_1^{(0)}+w_2^{(0)}}{2}}},
    \end{equation}
    and multiply by it since it is strictly positive. 
    Subtracting \eqref{eq:b_1_ubound} from \eqref{eq:b_2_lbound} yields
    \[
        b_2^{(t)} - b_1^{(t)} \ge t\cdot\frac{1}{2}\eta \exp\p{\frac{-w_1^{(0)}+w_2^{(0)}}{2}} - t\cdot\frac{1}{2}\eta \exp\p{-w_1^{(0)}}.
    \]
    Substituting $t$ with $t_0$, we have
    \begin{equation}\label{eq:bias_dif_lbound}
        b_2^{(t_0)} - b_1^{(t_0)} \ge \frac{1}{2}\eta t_0 \exp\p{\frac{-w_1^{(0)}+w_2^{(0)}}{2}} - \frac{1}{2}\eta t_0 \exp\p{-w_1^{(0)}}.
    \end{equation}
    Lower bounding the first summand in the right-hand side of the above using \eqref{eq:t_0_bounds}, we have
    \begin{equation}\label{eq:lb_first_summand}
        \frac12\eta t_0\exp\p{\frac{-w_1^{(0)}+w_2^{(0)}}{2}} \ge \frac{w_1^{(0)}+w_2^{(0)}}{4} - \frac12\eta\exp\p{\frac{-w_1^{(0)}+w_2^{(0)}}{2}} \ge \frac{w_1^{(0)}+w_2^{(0)}}{4} - 0.1,
    \end{equation}
    where in the last inequality, we used $\eta\le0.3$ and $-w_1^{(0)}+w_2^{(0)}<1$ which hold by our assumptions.

    We now lower bound the second summand in \eqref{eq:bias_dif_lbound}. We have
    \[
        -\frac12\eta t_0\exp\p{-w_1^{(0)}} \ge -\frac{w_1^{(0)}+w_2^{(0)}}{4}\exp\p{-\frac{w_1^{(0)}+w_2^{(0)}}{2}} \ge -\frac{w_1^{(0)}+w_2^{(0)}}{4}\exp\p{-\frac12}.
    \]
    Plugging the above and \eqref{eq:lb_first_summand} in \eqref{eq:bias_dif_lbound} yields the lower bound
    \[
        b_2^{(t_0)} - b_1^{(t_0)} \ge \frac{w_1^{(0)}+w_2^{(0)}}{4}\p{1-\exp\p{-\frac12}} -0.1 \ge \frac{w_1^{(0)}+w_2^{(0)}}{11}-\frac{1}{10}.
    \]
    Finally, we use the monotonicity of $b_2^{(t)} - b_1^{(t)}$ which holds by our assumption on the step size $\eta$ and \lemref{lem:monotone_limit}, to conclude the proof of the proposition for all $t\ge t_0$.

    \begin{proof}[Proof of \thmref{thm:non_asymptotic}]
        The theorem follows in a straightforward manner, by utilizing \propref{prop:non_asymptotic} to lower bound the numerator of $x^{\star}(t)$, using \lemref{lem:weight_dif_ubound} to upper bound the denominator, and plugging the bounds in the theorem statement to obtain explicit constants.
    \end{proof}
\end{document}